\documentclass[accepted]{uai2026} 
                        
\usepackage[british]{babel}
\usepackage{comment}
\usepackage{tabularx}

\usepackage{amsthm}

\newtheorem{theorem}{Theorem}

\newtheorem{lemma}{Lemma}
\newtheorem{corollary}{Corollary}
\newtheorem{definition}{Definition}
\newtheorem{assumption}{Assumption}
\newtheorem{remark}{Remark}

\usepackage{natbib} 
\usepackage{mathtools} 
\usepackage{booktabs} 
\usepackage{tikz} 

\usepackage{amssymb}

\providecommand{\E}{\mathbb{E}}
\providecommand{\KL}{\mathrm{KL}}

\providecommand{\cX}{\mathcal{X}}
\providecommand{\cY}{\mathcal{Y}}
\providecommand{\piref}{\pi_{\mathrm{ref}}}
\providecommand{\rholabel}{\rho_{\mathrm{label}}}
\providecommand{\Drm}{D_{\mathrm{train}}}
\providecommand{\Dtheta}{D_{\theta}}
\providecommand{\chisq}{\chi^2}
\providecommand{\Ccov}{\mathcal{C}}
\providecommand{\abs}[1]{\lvert #1\rvert}
\providecommand{\clip}{\operatorname{clip}}
\providecommand{\Cprompt}{\mathcal{C}_{\mathrm{prompt}}}
\providecommand{\Cpol}{\mathcal{C}_{\mathrm{pol}}}
\providecommand{\tr}{\operatorname{tr}}
\providecommand{\logdet}{\log\!\det}

\newcommand{\fh}[1]{{\color{purple}{[FH: #1]}}}

\title{Generalisation of RLHF under Reward Shift and Clipped KL Regularisation}

\author[1]{Kenton Tang}
\author[2]{Yuzhu Chen}
\author[1]{Fengxiang He}
\affil[1]{%
    University of Edinburgh
}
\affil[2]{%
    University of Science and Technology of China
}

\begin{document}
\maketitle
\begin{abstract}
Alignment and adaptation in large language models heavily rely on reinforcement learning from human feedback (RLHF); yet, theoretical understanding of its generalisability remains premature, especially when the learned reward could shift, and the KL control is estimated and clipped. To address this issue, we develop generalisation theory for RLHF that explicitly accounts for (1) \emph{reward shift}: reward models are trained on preference data from earlier or mixed behaviour policies while RLHF optimises the current policy on its own rollouts; and (2) \emph{clipped KL regularisation}: the KL regulariser is estimated from sampled log-probability ratios and then clipped for stabilisation, resulting in an error to RLHF. We present generalisation bounds for RLHF, suggesting that the generalisation error stems from a sampling error from prompts and rollouts, a reward shift error, and a KL clipping error. We also discuss special cases of (1) initialising RLHF parameters with a uniform prior over a finite space, and (2) training RLHF by stochastic gradient descent, as an Ornstein-Uhlenbeck process. The theory yields practical implications in (1) optimal KL clipping threshold, and (2) budget allocation in prompts, rollouts, and preference data.
\end{abstract}


\section{Introduction}
\label{sec:intro}

Reinforcement learning from human feedback (RLHF) 
has become a central method for steering large language models (LLMs) towards better reflecting human preferences \citep{christiano2017deep,stiennon2020learning}, task requirements \citep{ouyang2022training,chung2022scaling}, safety constraints \citep{bai2022training,bai2022constitutional}, amongst many others. 
Despite the empirical success of RLHF, the theoretical understanding of its generalisability remains largely absent. 

To address this issue, this paper presents generalisation bounds for RLHF. To note, we analyse the post-trained policy in deployment, rather than studying online reinforcement learning during interactions with the environment.

A typical RLHF algorithm consists of two coupled components: (1) a reward model trained from preference data and serving as a proxy for human judgment, and (2) a policy model optimised to maximise the reward model. Most RLHF algorithms additionally employ Kullback-Leibler (KL) regularisation to keep the policy close to a reference model, typically, a supervised fine-tuned (SFT) model \citep{ziegler2019fine}, for improving stability and limiting distribution shift {\citep{schulman2015trpo}}. 
These induce two major challenges that significantly complicate analysis, as follows.

\emph{Reward shift.} The reward model is usually trained on preference data collected from an \emph{earlier} behaviour policy or a \emph{mixture} of policies {\citep{christiano2017deep}}. 
However, the policy model is evaluated and optimised on rollouts drawn from the \emph{current} distribution of responses. As the policy improves or drifts, it can move into regions where the reward model is less reliable, creating a feedback loop in which reward-model error is amplified in optimisation \citep{gao2022scaling}. 
This calls for the potential RLHF generalisation theory to account for \emph{reward shift} between the data used to train the reward model and 
the rollout distribution induced by the current policy.

\emph{Clipped KL regularisation.} 
KL regularisation is usually 
assumed to be computed as a population expectation in theoretical treatments {\citep{schulman2015trpo}}. 
In practice, however, the KL control is computed from sampled sequences through log probability ratios; empirical analyses have shown that the choice of estimator and implementation details can materially affect optimisation stability \citep{shah2025comedy}. A common stabilisation is to clip the per-sample log ratio, 
in order to control rare trajectories whose likelihood ratios are extreme, echoing clipping used in PPO \citep{schulman2017ppo,lambert2025rlhfbook}. This clipped KL regularisation further introduces an error.

Motivated by these, we develop generalisation theory for RLHF that explicitly accounts for both: the reward is \emph{learned} and \emph{shifting}, rather than given and fixed; and the KL regularisation is \emph{estimated} and \emph{clipped}, rather than an exact population quantity. Based on a change-of-measure decomposition and employing PAC-Bayes tools \citep{mcallester1999modelavg,seeger2002pac,catoni2007pac}, our analysis 
yields high-probability generalisation bounds for the learned, data-dependent policy that decompose the generalisation error into three distinct, interpretable sources:
(1) a \emph{sampling error}, 
induced by the two-stage sampling of observing finitely many prompts and estimating expectations from limited Monte Carlo rollouts,
(2) a \emph{reward shift error}, capturing the gap between the learned reward and the (implicit) target reward, and the additional error induced when the policy-driven rollout distribution differs from the reward model's training distribution,
(3) a \emph{KL clipping error}, characterising the deviation from the clipped KL regulariser. 

{A good theory has practical implications. Our theory suggests:
(1) \emph{optimal KL clipping threshold}:
the theory indicates that the KL log-ratio clipping threshold $\tau$ controls the bias-variance trade-off, since clipping reduces sampling noise while introducing an objective mismatch that does not vanish asymptotically. Our theory further provides advice on how to strike a good balance;
(2) \emph{budget allocation across prompts, rollouts, and preference labels}:
our generalisation bounds separate the impacts of prompts, rollouts per prompt, and preference labels, thereby guiding budget allocation across prompts and rollouts, and preference data collection.
}

\section{Related Work}
\label{sec:related}

\paragraph{Optimisation theory of RLHF}
Efforts have been made in theoretically studying RLHF from an optimisation perspective. 
\citet{zhu2023principled} analyse RLHF based on pair-wise and list-wise comparisons, and characterise how the reward model error can induce suboptimal policies, motivating conservative strategies under coverage assumptions.
Similarly, \citet{zhan2023provable} provide finite-sample guarantees for offline RLHF that depend on a concentrability coefficient quantifying the coverage of the target policy by the offline data.
\citet{xiong24a} establish finite-sample guarantees for KL-regularised RLHF, in the offline, online, and hybrid regimes.

\paragraph{Reward shift}
The literature has seen empirical studies on the impact of reward shift.
\citet{gleave2022uncertainty} empirically study uncertainty estimation for reward models, highlighting that reward models can be unreliable out of distribution.
{\citet{gao2022scaling} empirically characterise reward model over-optimisation by measuring how the proxy-oracle gap grows when a policy is optimised against a learned proxy reward and evaluated under a stronger oracle reward.}
In addition, RewardBench provides a complementary evaluation resource for quantifying reward model behaviour on challenging and out-of-distribution comparisons \citep{lambert2024rewardbench}.

\paragraph{Clipped KL regularisation}
As an empirical work, \citet{shah2025comedy} provide an extensive analysis showing that several commonly used estimators for KL regularisation can produce biased gradients, which can affect optimisation and stability.
\citet{liu2025rethinkingkl} analyse KL regularisation implementations in RLHF and characterise when common choices are principled or biased, including off-policy bias that arises when importance weighting is neglected.

\paragraph{Concurrent paper}
A concurrent work, released on 23 Jan 2026, provides interesting results on the generalisation of RLHF, 
under linear reward model assumption, through the algorithmic stability framework \citep{li2026generalization}.
Our work is more general, 
formulated for RLHF pipelines beyond linear reward; 
instead, the reward in this paper is learned from preference data and shifts with policy updates. Moreover, our paper allows the KL control to be estimated from sampled log ratios and clipped for stabilisation, 
{while \citet{li2026generalization} formulate the KL penalty as an exact conditional KL divergence term in the objective, without sample-based KL estimation or clipping.}
\section{Preliminaries}
\label{sec:pre}

\paragraph{RLHF} 
Given a prompt $x\in\cX$, a policy $\pi$ specifies a conditional distribution $\pi(\cdot\mid x)$ over responses $y\in\cY$.
We denote the post-trained policy by $\pi_\theta$ with parameter $\theta\in\Theta$, and denote $\piref$ as a fixed reference policy.
Evaluation uses prompts drawn from a distribution $\rho$, while preference data for reward modelling may come from a different prompt distribution $\rholabel$ because of prompt shift.

Suppose the target reward is $r^\star:\cX\times\cY\to[0,1]$.
A reward model is a proxy $\hat r_\phi:\cX\times\cY\to[0,1]$, with parameter $\phi\in\Phi$; the pointwise error is $e_\phi(x,y)=\hat r_\phi(x,y)-r^\star(x,y)$.
{Training the reward model uses a data collection distribution, defined as $D_{\mathrm{train}}(x,y) = \rholabel(x)\pi_{\mathrm{ref}}(y\mid x)$, where we also use $\piref$ as the behaviour policy for reward-data collection. In practice, this policy can be a mixture, and we write
$\pi_{\mathrm{ref}}(\cdot\mid x) := \sum_{m=1}^{M} c_{m}\pi^{(m)}(\cdot\mid x)$, reflecting the standard practice of collecting preference rankings over diverse behaviour policy mixtures.}
Moreover, the policy-induced distribution is defined as $D_\theta(x,y)=\rho(x)\,\pi_\theta(y\mid x)$.

The reward model is evaluated on the training distribution by the mean-squared error $L_{\mathrm{train}}^{(2)}(\phi)$, defined by
\begin{equation}
\label{eq:pre:l2train}
\begin{aligned}
\E_{(X,Y)\sim D_{\mathrm{train}}}\!\left[\bigl(\hat r_\phi(X,Y)-r^\star(X,Y)\bigr)^2\right].
\end{aligned}
\end{equation}


This quantity is an oracle-risk term defined with respect to $r^\star$, and is typically not directly observable from pairwise preference labels in practice.

\paragraph{Clipped KL regularisation}
Let $\beta>0$ denote the regularisation strength, and \(\ell_\theta(x,y)=\log \pi_\theta(y\mid x)-\log \piref(y\mid x)\) denote the exact log ratio. This log ratio is the per-sample quantity that appears when the KL control 
is implemented from sampled rollouts, which refer to the response sequence generated by sampling sequentially from the policy conditioned on the prompt. Its conditional expectation strictly recovers the standard reference KL divergence (Definition~\ref{def:app:kl}, Appendix~\ref{app:tools}) within the population objective.
In particular, for every prompt $x$, we have
\(
\E_{Y\sim\pi_\theta(\cdot\mid x)}\!\left[\ell_\theta(x,Y)\right]
=
\KL\!\left(\pi_\theta(\cdot\mid x)\,\middle\|\,\piref(\cdot\mid x)\right)
\).

In post-training, $\ell_\theta(x,y)$ can have a large magnitude on rare samples, which can significantly increase the variance of empirical KL-related quantities and destabilise optimisation unless additional control is imposed \citep{shah2025comedy,lambert2025rlhfbook}.
To stabilise KL control while keeping the target objective explicit, a popular approach is clipping with threshold $\tau>0$: $\ell_\theta^\tau(x,y)=\clip(\ell_\theta(x,y),-\tau,\tau)$ \citep{schulman2017ppo,lambert2025rlhfbook}. 
Correspondingly, the clipped population objective $J^{r,\tau}(\theta)$ is given by
\begin{equation}
\label{eq:pre:J_pop}
\begin{aligned} 
\E_{X\sim\rho}\E_{Y\sim\pi_\theta(\cdot\mid X)}
\!\left[
r(X,Y)
-\beta\,\ell_\theta^\tau(X,Y)
\right].
\end{aligned}
\end{equation}

\paragraph{Generalisation} 
The population objective is
\[
J^{r}(\theta)
=
\E_{X\sim\rho}\E_{Y\sim\pi_\theta(\cdot\mid X)}
\!\left[
r(X,Y)
-\beta\,\ell_\theta(X,Y)
\right],
\]
where $\ell_\theta(x,y)=\log\pi_\theta(y\mid x)-\log\piref(y\mid x)$ is the exact log ratio.
Evaluating a policy relies on finite prompts and rollouts.
Let $x_1,\dots,x_n$ be independent prompts drawn from $\rho$.
For each $x_i$, let $y_{i,1},\dots,y_{i,K}$ denote $K$ independent rollouts drawn from $\pi_\theta(\cdot\mid x_i)$.
The resulting empirical objective is
\[
\widehat J_{n,K}^{r,\tau}(\theta)
=
\frac{1}{n}\sum_{i=1}^n
\frac{1}{K}\sum_{j=1}^K
\left[
r(x_i,y_{i,j})
-\beta\,\ell_\theta^\tau(x_i,y_{i,j})
\right].
\] 
{For brevity, $\widehat J_{n,K}^{\phi,\tau}(\theta)$ denotes $\widehat J_{n,K}^{\hat r_\phi,\tau}(\theta)$; 
$J^\star(\theta)$ denotes $J^{r^\star}(\theta)$;
$J^\phi(\theta)$ denotes $J^{\hat r_\phi}(\theta)$;
$J^{\phi,\tau}(\theta)$ denotes $J^{\hat r_\phi,\tau}(\theta)$.} 

The generalisability can be quantified by the generalisation error, defined to be the discrepancy between the empirical and population objectives: $\left|
\widehat J_{n,K}^{\phi,\tau}(\theta)-J^\star(\theta)
\right|$.

\section{Main Results}
\label{sec:main}

This section presents our theoretical results. 

\subsection{Decomposing Generalisation Error}
\label{subsec:main:decomp}


{We decompose the generalisation error into three components: (1) a \emph{sampling error}, induced by prompts and rollouts, which is present even if the following two terms do not exist; (2) a \emph{reward shift error}, induced by reward shift under the same \emph{exact} KL regulariser; and (3) a \emph{KL clipping error}, induced by the objective mismatch induced by estimating and clipping the KL penalty.}

\begin{lemma}[Generalisation error decomposition]
\label{prop:main:decomp}
Given parameters $\theta\in\Theta$ and $\phi\in\Phi$, and clipping threshold $\tau>0$.
Then, we have the following decomposition,
\begin{equation}
\label{eq:main:decomp}
\begin{aligned}
&\left|
\widehat J_{n,K}^{\phi,\tau}(\theta)-J^\star(\theta)
\right|\\
\le&
\underbrace{\left|
\widehat J_{n,K}^{\phi,\tau}(\theta)-J^{\phi,\tau}(\theta)
\right|}_{\text{sampling error}}
+
\underbrace{\left|
J^\phi(\theta)-J^\star(\theta)
\right|}_{\text{reward shift error}}\\
&+
\underbrace{\left|
J^{\phi,\tau}(\theta)-J^\phi(\theta)
\right|}_{\text{KL clipping error}}
\end{aligned}
\end{equation}
\end{lemma}

\subsection{Sampling Error Bound}
\label{subsec:main:sampling}

We first study the sampling error
$
\left|
\widehat J_{n,K}^{\phi,\tau}(\theta)-J^{\phi,\tau}(\theta)
\right|
$. 
We define $\widehat J_{n,\infty}^{r,\tau}(\theta)$ as the conditional expectation of $\widehat J_{n,K}^{r,\tau}(\theta)$, given the prompts $x_{1:n}$.
Equivalently, it is the value one would obtain by averaging infinitely many rollouts per prompt while keeping the same finite set of prompts.
The estimator $\widehat J_{n,K}^{\phi,\tau}(\theta)$ thus has a two-stage structure: 
(1) prompts are sampled from $\rho$, leading to a deviation $\left|
\widehat J_{n,\infty}^{r,\tau}(\theta)-J^{r,\tau}(\theta)
\right|$, and 
(2) rollouts are sampled from $\pi_\theta(\cdot\mid x)$, conditional on each prompt, inducing a deviation $\left|
\widehat J_{n,K}^{r,\tau}(\theta)-\widehat J_{n,\infty}^{r,\tau}(\theta)
\right|$. 

We first bound the rollout sampling error as follows.

\begin{lemma}[Rollout sampling error bound]
\label{lem:main:rollout}
Given parameter $\theta\in\Theta$, reward $r:\cX\times\cY\to[0,1]$, clipping threshold $\tau>0$, and confidence level $\delta\in(0,1)$,
with probability at least $1-\delta$ over rollouts, conditional on prompts $x_{1:n}$, we have
\[
\left|
\widehat J_{n,K}^{r,\tau}(\theta)-\widehat J_{n,\infty}^{r,\tau}(\theta)
\right|
\le
\left(1+2\beta\tau\right)
\sqrt{\frac{\log(2/\delta)}{2nK}}.
\]
\end{lemma}


\paragraph{Proof sketch} Loss clipping ensures that $\ell_\theta^\tau(x,y)\in[-\tau,\tau]$ by construction, which is a standard stabilisation approach in reinforcement learning \citep{mnih2015dqn,schulman2017ppo}. Combining that the reward function satisfies $r(x,y)\in[0,1]$,
for each per-rollout contribution, 
$
r(x,y)-\beta\,\ell_\theta^\tau(x,y)
$
lies in an interval of length $1+2\beta\tau$. 
Given the prompts $x_{1:n}$, the rollouts are independent across both $i$ and $j$.
Applying Hoeffding's inequality ({Lemma~\ref{lem:app:hoeffding}}) to the average over the $nK$ rollout terms yields Lemma~\ref{lem:main:rollout}. 
Detailed proofs are in Appendix~\ref{app:proofs:stat}.


\begin{remark}
Lemma~\ref{lem:main:rollout} controls the Monte Carlo error from using finitely many rollouts per prompt.
The bound decays at {rate $O(nK)^{-1/2}$} as the number of rollouts per prompt $K$ increases (assuming $\beta$ and $\tau$ are independent of $n$ and $K$).
The factor $1+2\beta\tau$ comes from the range of the per-rollout contribution.
Clipping is the mechanism that makes this range finite without imposing any artificial uniform bound on the exact log ratio $\ell_\theta$.
\end{remark}

\begin{lemma}[Prompt sampling error bound]
\label{lem:main:prompt}
Under the same conditions of Lemma \ref{lem:main:rollout}, with probability at least $1-\delta$ over prompts $x_{1:n}$, we have
\[
\left|
\widehat J_{n,\infty}^{r,\tau}(\theta)-J^{r,\tau}(\theta)
\right|
\le
\left(1+2\beta\tau\right)
\sqrt{\frac{\log(2/\delta)}{2n}}.
\]
\end{lemma}

\paragraph{Proof sketch} Treating $\widehat J_{n,\infty}^{\phi,\tau}(\theta)$ as a function of the sampled prompts only, it is an average of $n$ independent bounded terms, each term being the conditional expectation over rollouts for a fixed prompt.
Applying Hoeffding's inequality again yields Lemma~\ref{lem:main:prompt}. 
Detailed proofs are in Appendix~\ref{app:proofs:stat}.

\begin{remark}
Lemma~\ref{lem:main:prompt} suggests that the prompt sampling error decays at {rate $O(n^{-1/2})$}, and the corresponding bound does not depend on the number of rollouts per prompt $K$, similarly, assuming $\beta$ and $\tau$ are independent of $n$ and $K$.
It isolates the deviation induced purely by finite prompt sampling.
Even an arbitrarily accurate estimate of each conditional expectation over rollouts cannot compensate for having too few evaluation prompts, because the population objective is defined as an expectation over $\rho$.
\end{remark}

Combining the two lemmas leads to the following lemma on the sampling error.

\begin{lemma}[Sampling error bound]
\label{lem:main:stat}
Under the same conditions of Lemma \ref{lem:main:rollout}, with probability at least $1-\delta$ over both prompts and rollouts, the sampling error satisfies 
\begin{equation}
\label{eq:main:err_stat}
\begin{aligned}
&\left|
\widehat J_{n,K}^{r,\tau}(\theta)-J^{r,\tau}(\theta)
\right| \\
&\qquad\le
(1+2\beta\tau)
\left(
\sqrt{\frac{\log(4/\delta)}{2n}}
+
\sqrt{\frac{\log(4/\delta)}{2nK}}
\right).
\end{aligned}
\end{equation}
\end{lemma}


\begin{remark}
In addition to the noise induced by prompts and rollouts, a penalty term carries the additional factor $2\beta\tau$ because the clipped log ratio ranges in $[-\tau,\tau]$. 
Consequently, increasing $\tau$ enlarges the range of each rollout penalty term, and the resulting concentration bound is looser.
\end{remark}

\subsection{Reward shift error bound}
\label{subsec:main:shift}

This subsection studies the reward shift error
$
\left|J^\phi(\theta)-J^\star(\theta)\right|
$.
To characterise the reward shift error in transfers from $D_{\mathrm{train}}$ to $D_\theta$, we use a $\chisq$ coverage coefficient, defined below, based on {$\chisq$ divergence (see Definition~\ref{def:app:chisq} in Appendix~\ref{app:tools})}. $\chisq$ coverage coefficient is standard in the literature of importance weighting and covariate shift analyses; see, e.g.,  \citet{sugiyama2007covariate,owen_chisq_notes}.

\begin{definition}[$\chisq$ coverage coefficient]
Suppose that $D_\theta$ is absolutely continuous with respect to $D_{\mathrm{train}}$.
The $\chisq$ coverage coefficient is defined to be
\begin{equation}
\label{eq:pre:ccov}
\Ccov(\theta)
:=
\sqrt{1+\chisq(D_\theta\|D_{\mathrm{train}})},
\end{equation}
where $\chisq(\cdot\|\cdot)$ is $\chisq$ divergence.
\end{definition} 

\begin{remark}
Intuitively, $\Ccov(\theta)$ measures how far the policy-induced distribution departs from the distribution used to train the reward model. It acts as an amplification factor when we upper bound the reward shift error. 
\end{remark}

Because $J^\phi(\theta)$ and $J^\star(\theta)$ share the same exact KL regulariser, the KL penalty cancels in the difference; consequently, only the reward model error remains.
Defining $e_\phi(x,y)=\hat r_\phi(x,y)-r^\star(x,y)$, we have
\(
J^\phi(\theta)-J^\star(\theta)
=
\E_{(X,Y)\sim D_\theta}\!\left[e_\phi(X,Y)\right]\),
so the problem is to control the reward model error under the deployment distribution $D_\theta$ using information available under the training distribution $D_{\mathrm{train}}$. 
{This step requires a coverage condition, stated below, when deriving the change-of-measure bound; it yields the same coefficient $\Ccov(\theta)$ defined in eq.~\eqref{eq:pre:ccov}.}

\begin{assumption}[Absolute continuity and finite coverage]
\label{assump:main:coverage}
The policy-induced distribution $D_\theta$ is absolutely continuous with respect to the reward model training distribution $D_{\mathrm{train}}$.
Moreover, the $\chisq$ divergence $\chisq(D_\theta\|D_{\mathrm{train}})$ is finite, so the coverage coefficient $\Ccov(\theta)$ in eq. \eqref{eq:pre:ccov} is finite.
\end{assumption}

\begin{remark}
Assumption~\ref{assump:main:coverage} is the standard condition that makes a change of measure from $D_\theta$ back to $D_{\mathrm{train}}$ legitimate \citep{sugiyama2007covariate,shimodaira2000covariate,precup2000eligibility}.
It ensures that the density ratio $\frac{\mathrm{d}D_\theta}{\mathrm{d}D_{\mathrm{train}}}$ exists and has a finite second moment, which is required for the Cauchy-Schwarz step in Lemma~\ref{lem:main:shift} \citep{owen_chisq_notes}.
The coefficient $\Ccov(\theta)$ plays the role of an amplification factor, which quantifies how strongly the reward model error can be magnified when the policy visits regions that are rare under the data used for reward modelling.
\end{remark}

Our theory also relies on the following mild assumption.

\begin{assumption}[Bounded training error]
\label{assump:main:rm}
The squared training error $L_{\mathrm{train}}^{(2)}(\phi)$ defined in eq. \eqref{eq:pre:l2train} is bounded.
\end{assumption}

\begin{remark}
This assumption does not assert that the reward model is accurate everywhere; instead, it provides a baseline level of accuracy on the distribution where preference supervision is available. The coverage coefficient explains how the baseline can degrade under deployment.
\end{remark}

We then prove the reward shift bound. 

\begin{lemma}[Reward shift error bound]
\label{lem:main:shift}
Under Assumptions~\ref{assump:main:coverage} and \ref{assump:main:rm}, we have
\[
\left|
J^\phi(\theta)-J^\star(\theta)
\right|
\le
\Ccov(\theta)\,
\sqrt{L_{\mathrm{train}}^{(2)}(\phi)}.
\]
\end{lemma}

\paragraph{Proof sketch} 

To relate
$
\left|
J^\phi(\theta)-J^\star(\theta)
\right|
$
to the training distribution, we rewrite the expectation under $D_\theta$ as an importance-weighted expectation under $D_{\mathrm{train}}$.
Assumption~\ref{assump:main:coverage} ensures that the required density ratio exists and has finite second moment.
Applying the $\chisq$ change-of-measure bound (Lemma~\ref{lem:app:chisq_com} in Appendix~\ref{app:tools}) yields a product of two square roots:
{the first factor under the square root is exactly $1+\chisq(D_\theta\|D_{\mathrm{train}})$, the second moment of the density ratio under $D_{\mathrm{train}}$, whose square root therefore produces $\Ccov(\theta)$; and the second factor under the square root is $L_{\mathrm{train}}^{(2)}(\phi)$ by definition, the second moment of the reward model error under $D_{\mathrm{train}}$.
This yields Lemma~\ref{lem:main:shift}.
Detailed proofs are in Appendix~\ref{app:proofs:shiftbias}.}

\begin{remark}
Lemma~\ref{lem:main:shift} characterises two ingredients that play different roles:
(1) the term $L_{\mathrm{train}}^{(2)}(\phi)$ measures reward model error only on the reward model training distribution $D_{\mathrm{train}}$; and
(2) the coefficient $\Ccov(\theta)$ measures how far $D_\theta$ moves away from that training distribution. It also quantifies how much training error can be amplified when moving to deployment.
\end{remark}

When prompt shifts and policy shifts are qualitatively different, it is useful to further factorise the coverage coefficient.
The next lemma 
interprets the source of shifts in practice. 

\begin{lemma}[Coverage factorisation]
\label{lem:main:cov_factor}
Suppose $\rho\ll\rholabel$ and $\pi_\theta(\cdot\mid x)\ll\piref(\cdot\mid x)$, for any prompt $x$ with $\rholabel(x)>0$.
Define
\[
\Cprompt
=
\left(
\E_{X\sim\rholabel}
\left[
\left(\frac{\rho(X)}{\rholabel(X)}\right)^2
\right]
\right)^{1/2},
\]
and define $\Cpol(\theta)$ by
\[
\begin{aligned}
\sup_{x\in\cX:\rholabel(x)>0}
\left(
\E_{Y\sim\piref(\cdot\mid x)}
\left[
\left(\frac{\pi_\theta(Y\mid x)}{\piref(Y\mid x)}\right)^2
\right]
\right)^{1/2}
\end{aligned}.
\]
If both $\Cprompt$ and $\Cpol(\theta)$ are bounded, we have $\Ccov(\theta)\le \Cprompt\,\Cpol(\theta)$.
\end{lemma}


\begin{remark}
Lemma~\ref{lem:main:cov_factor} separates mismatch in prompts from mismatch in policies.
The coefficient $\Cprompt$ depends only on the shift between $\rho$ and $\rholabel$.
The coefficient $\Cpol(\theta)$ depends only on how far $\pi_\theta$ departs from $\piref$ on the support of $\rholabel$.
This separation is valuable when diagnosing failures in reward modelling and post training, because the two sources of shift have different operational causes and different mitigation strategies.
\end{remark}

\subsection{KL Clipping Error Bound}
\label{subsec:main:bias}

We now bound the KL clipping error
$
\left|J^{\phi,\tau}(\theta)-J^\phi(\theta)\right|.
$
This is the only place where the systematic mismatch created by clipping enters the analysis.
Clipping is beneficial in the sampling bounds because it makes each rollout contribution bounded.
Meanwhile, clipping may bias the objective used in practice, thereby introducing a systematic mismatch between the optimised objective and the target objective.

To state this mismatch cleanly, we only require an integrability condition on the exact log ratio in deployment.

\begin{assumption}[Integrability of exact log ratio]
\label{assump:main:logratio}
The exact log ratio is integrable in deployment, i.e.,
$\E_{(X,Y)\sim D_\theta}\!\left[|\ell_\theta(X,Y)|\right]<\infty$.
\end{assumption}

\begin{remark}
Assumption~\ref{assump:main:logratio} is mild.
It allows heavy tails in $\ell_\theta$ while still ensuring that the exact objective is well defined.
Under this assumption, clipping is analysed as an explicit bias-inducing modification of the penalty; 
{the KL clipping error term in Lemma~\ref{lem:main:bias},
\(
\beta\,
\E_{(X,Y)\sim D_\theta}
\!\left[
\left|
\ell_\theta(X,Y)-\ell_\theta^\tau(X,Y)
\right|
\right],
\)
is an objective mismatch term that does not vanish asymptotically as the number of evaluation prompts or rollouts increases.}
It is strictly weaker than assuming $\ell_\theta$ is uniformly bounded, and it matches the intent of treating clipping as an algorithmic choice rather than as a structural property of the policy class. 
{Similar integrability conditions are standard in analysing truncation-based stabilisation and importance sampling; see, e.g., \citet{ionides2008truncated,owen2000safeis}.}
\end{remark}

We then prove the following bound on the bias induced by the surrogate.

\begin{lemma}[KL clipping error bound]
\label{lem:main:bias}
Under Assumption~\ref{assump:main:logratio}, we have
\begin{equation}
\label{eq:main:bias}
\begin{aligned}
&\left|
J^{\phi,\tau}(\theta)-J^\phi(\theta)
\right| \\
&\qquad\le
\beta\,
\E_{(X,Y)\sim D_\theta}
\!\left[
\left|
\ell_\theta(X,Y)-\ell_\theta^\tau(X,Y)
\right|
\right].
\end{aligned}
\end{equation}
\end{lemma}

\paragraph{Proof sketch} 
$
\left|
J^{\phi,\tau}(\theta)-J^\phi(\theta)
\right|
$
is not an estimation error; it compares two population objectives under the same learned reward, with the only difference being whether the penalty uses $\ell_\theta^\tau$ or $\ell_\theta$. Expanding definitions shows that the reward contributions cancel and only the penalty difference remains.
Taking absolute values and applying the triangle inequality yields Lemma~\ref{lem:main:bias}. 
Detailed proofs are in Appendix~\ref{app:proofs:shiftbias}.

\begin{remark}
The right-hand side of eq. \eqref{eq:main:bias} measures clipping bias directly as the expected amount of truncation under the deployment distribution.
This term can remain nonzero even with infinite evaluation data, which reflects the fact that clipping is an objective mismatch rather than an estimation error.
It is small when the policy rarely produces extreme log ratios under $D_\theta$, and it can be large when the policy places substantial mass in regions where the exact log ratio has heavy tails. This is why the final bound contains a term that depends on the tail behaviour of the exact log ratio under $D_\theta$ and does not involve $n$ or $K$.
\end{remark}

\subsection{Fixed-Policy generalisation Bound}
\label{subsec:main:fixed}

We now combine the results on sampling error, reward shift error, and KL clipping error into a single statement for a fixed policy parameter $\theta$, as follows.

\begin{theorem}[Fixed-policy generalisation bound]
\label{thm:main:fixed}
Under Assumptions~\ref{assump:main:coverage}, \ref{assump:main:rm}, and \ref{assump:main:logratio},
with probability at least $1-\delta$ over the evaluation prompts and rollouts, the following holds,
\begin{equation}
\label{eq:main:fixed}
\begin{aligned}
&\left|
\widehat J_{n,K}^{\phi,\tau}(\theta)-J^\star(\theta)
\right| \\
&\qquad\le
\underbrace{
(1+2\beta\tau)
\left(
\sqrt{\frac{\log(4/\delta)}{2n}}
+
\sqrt{\frac{\log(4/\delta)}{2nK}}
\right)
}_{\text{sampling error}}
\\
&\qquad +
\underbrace{
\Ccov(\theta)\,
\sqrt{L_{\mathrm{train}}^{(2)}(\phi)}
}_{\text{reward shift error}}
\\
&\qquad +
\underbrace{
\beta\,
\E_{(X,Y)\sim D_\theta}
\!\left[
\left|
\ell_\theta(X,Y)-\ell_\theta^\tau(X,Y)
\right|
\right]
}_{\text{KL clipping error}}.
\end{aligned}
\end{equation}
\end{theorem}

\subsection{Data-Dependent PAC-Bayes Bound}
\label{subsec:main:pac}

The fixed-policy theorem treats $\theta$ as pre-fixed.
In practice, $\theta$ is often selected after observing data. 
This section fixes the gap by employing PAC-Bayes theory that extends our analysis to data-dependent selection {\citep{mcallester1999modelavg,seeger2002pac,catoni2007pac}}. 
Specifically, we provide a bound that holds simultaneously for all posteriors $Q$ over $\Theta$, at the cost of a complexity term that measures how far $Q$ deviates from a prior $P$ on $\Theta$.
Define
\[
\widehat J_{n,K}^{\phi,\tau}(Q)
=
\E_{\theta\sim Q}
\!\left[
\widehat J_{n,K}^{\phi,\tau}(\theta)
\right],\quad
J^\star(Q)
=
\E_{\theta\sim Q}
\!\left[
J^\star(\theta)
\right].
\] 
Then, we have the following data-dependen PAC-Bayes bound.

\begin{theorem}[Data-dependent generalisation bound]
\label{thm:main:pac}

Let $P$ be any prior distribution on $\Theta$ that is independent of the evaluation prompts and rollouts.
For any posterior $Q$ on $\Theta$, suppose Assumptions~\ref{assump:main:coverage}, \ref{assump:main:rm}, and \ref{assump:main:logratio} hold for any $\theta$ in the support of $Q$.
Then, with probability at least $1-\delta$ over the evaluation prompts and rollouts, the following inequality holds simultaneously for all such posteriors $Q$,
\[
\begin{aligned}
&\left| \widehat J_{n,K}^{\phi,\tau}(Q)-J^\star(Q) \right|\\
&\qquad\le
\underbrace{%
(1+2\beta\tau)
\begin{aligned}[t]
&\Bigg(\sqrt{\frac{\KL(Q\|P)+\log(8/\delta)}{2n}}\\
&\quad+
\sqrt{\frac{\KL(Q\|P)+\log(8/\delta)}{2nK}}\Bigg)
\end{aligned}
}_{\text{sampling error}}
\\
&\qquad+
\underbrace{%
\E_{\theta\sim Q}\!\left[\Ccov(\theta)\right]\,
\sqrt{L_{\mathrm{train}}^{(2)}(\phi)}%
}_{\text{reward shift error}}
\\
&\qquad+
\underbrace{%
\beta\,\E_{\theta\sim Q}\!\left[
\E_{(X,Y)\sim D_\theta}\!\left[
\left| \ell_\theta(X,Y)-\ell_\theta^\tau(X,Y) \right|
\right]\right]%
}_{\text{KL clipping error}}.
\end{aligned}
\]
\end{theorem}

\begin{remark}
Comparing with Theorem~\ref{thm:main:fixed}, Theorem~\ref{thm:main:pac} replaces the fixed $\theta$ with an average over $\theta\sim Q$.
The complexity term $\KL(Q\|P)$ appears only inside the sampling error, as the price paid for making the guarantee hold uniformly over data-dependent choices of $Q$.
\end{remark}

\section{Special cases} 
\label{subsec:main:complexity}

The PAC-Bayes bound in Theorem~\ref{thm:main:pac} contains a complexity term $\KL(Q\|P)$, which measures how strongly the data-dependent posterior $Q$ departs from the data-independent prior $P$.
This subsection discusses two operational instantiations of $\KL(Q\|P)$ that are common in practice.

\subsection{Initialisation by uniform prior over finite candidate class}


Let $M\ge 2$ be an integer,  $\theta^{(1)},\dots,\theta^{(M)}\in\Theta$ be a collection of candidate parameters specified independently of the evaluation sample used to construct $\widehat J_{n,K}^{\phi,\tau}$.
Suppose $\Theta_M:=\{\theta^{(1)},\dots,\theta^{(M)}\}$ and restrict both $P$ and $Q$ to be distributions on $\Theta_M$.
Suppose the prior is uniform on $\Theta_M$; i.e., $P(\theta^{(m)})=1/M$ for any $m$. 
This non-informative prior is standard in finite model selection \citep{seeger2002pac}.

\begin{corollary}[KL bound for uniform prior over finite candidate class]
\label{cor:main:finite_class}
Under the conditions above, 
$\KL(Q\|P)
\le
\log M$.
In particular, if $Q$ is the Dirac distribution supported on a data-selected checkpoint $\theta^{(\widehat m)}$, we have $\KL(Q\|P)=\log M$.
\end{corollary}


\begin{remark}
Corollary~\ref{cor:main:finite_class} yields a direct interpretation of model selection in the PAC-Bayes bound.
Evaluating $M$ fixed checkpoints and selecting one using the evaluation sample incurs an additional sampling error cost in Theorem~\ref{thm:main:pac}, controlled by $\log M$ via the quantity $\KL(Q\|P)$.
\end{remark}

\subsection{Training RLHF by SGD as Ornstein-Uhlenbeck process}

Stochastic gradient descent (SGD), and its variants, are popular optimisers  \citep{RobbinsMonro1951}.
Suppose the parameter space is $\mathbb{R}^d$ for some $d\in\mathbb{N}$. Assume the prior is Gaussian, i.e., $P=\mathcal{N}(\theta_0,\Lambda)$ for some $\theta_0\in\mathbb{R}^d$ and some symmetric positive definite matrix $\Lambda\in\mathbb{R}^{d\times d}$. 



We employ a standard local diffusion approximation for constant-step-size SGD. Near a locally stable optimum, late-stage SGD iterates follow an Ornstein-Uhlenbeck (OU) process \citep{Uhlenbeck1930}:
\[
d\theta_t
=
-H(\theta_t-\hat\theta)\,dt
+
\sqrt{\varepsilon}\,\Sigma_g^{1/2}\,dW_t,
\]
where $W_t$ is a $d$-dimensional Brownian motion, $H\succ 0$ is the local Hessian at $\hat\theta$, and $\Sigma_g\succ 0$ is the local gradient-noise covariance. {We make the following assumptions, standard for this local OU approximation; see \citet{Mandt2017,he2019control,chen2023stochastic}.}


\begin{assumption}    
Assume the optimisation problem has a locally stable optimum $\hat\theta\in\mathbb{R}^d$; i.e., within a neighbourhood of $\hat\theta$, the objective admits a quadratic approximation with Hessian $H\succ 0$ and the gradient noise covariance is approximately constant and equal to $\Sigma_g\succ 0$.
In addition, the matrices $H$ and $\Sigma_g$ commute; i.e., $H\Sigma_g=\Sigma_g H$ holds.
Moreover, there exist constants $0<m\le M<\infty$ such that the local curvature spectrum satisfies $mI\preceq H\preceq MI$.
\end{assumption}



Under the assumption above, the OU process admits a stationary Gaussian law \(\mathcal{N}(\hat\theta,\Sigma)\), where \(\Sigma\succ 0\) satisfies the continuous Lyapunov equation \(H\Sigma+\Sigma H=\varepsilon\Sigma_g\). Accordingly, we approximate the PAC-Bayes posterior induced by late-stage SGD iterates by
\(
Q_{\mathrm{SGD}}:=\mathcal{N}(\hat\theta,\Sigma).
\)




\begin{corollary}[KL bound for SGD as Ornstein-Uhlenbeck process]
\label{cor:main:ou_sgd}
In the special case above, 
the PAC-Bayes complexity term admits the upper bound
\begin{equation}
\label{eq:main:ou_sgd_kl}
\begin{aligned}
&\KL(Q_{\mathrm{SGD}}\|P)
\\ \le
\frac12\Big(
&(\hat\theta-\theta_0)^\top\Lambda^{-1}(\hat\theta-\theta_0)
+ \frac{\varepsilon}{2m}\tr(\Lambda^{-1}\Sigma_g)
-d
\\
&+ \logdet(\Lambda) - \logdet(\Sigma_g)
- d\log\Big(\frac{\varepsilon}{2M}\Big)
\Big).
\end{aligned}
\end{equation}
\end{corollary}

A detailed proof is given in Appendix~\ref{app:proofs:ou_sgd}.



\begin{remark}
Corollary~\ref{cor:main:ou_sgd} yields a locally valid, optimiser-explicit bound for $\KL(Q\|P)$ via the stationary diffusion approximation of constant-step-size SGD.
This secondary specialisation of Theorem~\ref{thm:main:pac} imposes no additional structural assumptions on the main RLHF analysis.
The diffusion perspective and the associated Ornstein-Uhlenbeck approximation are discussed in detail by \citet{Mandt2017}.
\end{remark}

\section{Practical Implications}

The discussion below translates our theory into concrete, practical algorithm design recommendations. 


\subsection{Optimal KL clipping threshold}

Lemma~\ref{lem:main:stat} includes 
a factor $1+2\beta\tau$, indicating that a smaller $\tau$ tightens the sampling deviations that arise from finite $n$ and $K$.
{Lemma~\ref{lem:main:kl_mc} in Appendix~\ref{app:tools} gives the corresponding KL-specific concentration bound for the clipped log-ratio average, whose deviation also scales linearly with $\tau$.}
Meanwhile, clipping changes the regularised objective and introduces a systematic mismatch that does not vanish with more evaluation samples, as formalised in Lemma~\ref{lem:main:bias}. 


Therefore, $\tau$ acts as a bias-variance trade-off hyperparameter rather than a purely stabilising tweak. Aggressive clipping reduces Monte Carlo noise but increases objective mismatch; weak clipping preserves the exact KL objective but exposes training to high-variance log-ratio estimates.

For brevity, we define
\[
\begin{aligned}
\alpha_{n,K,\delta} := &\sqrt{\frac{\log(4/\delta)}{2n}}+\sqrt{\frac{\log(4/\delta)}{2nK}},
\\
T_\theta(\tau) := &\E_{(X,Y)\sim D_\theta}\big[(|\ell_\theta(X,Y)|-\tau)_+\big].
\end{aligned}
\]
Since $|\ell_\theta-\ell_\theta^\tau|=(|\ell_\theta|-\tau)_+$ under symmetric clipping, the $\tau$-dependent part of eq. \eqref{eq:main:fixed} is thus
\[
B_\theta(\tau) := (1+2\beta\tau)\,\alpha_{n,K,\delta} \;+\; \beta\,T_\theta(\tau).
\]
Let $\tau^\star$ be any minimiser of $\tau\mapsto B_\theta(\tau)$ over $\tau\ge 0$.
We have the following corollary.

\begin{corollary}[Optimal KL clipping threshold]\label{cor:opt_tau}
For any parameters $\theta\in\Theta$ and $\phi\in\Phi$, 
regularisation coefficient $\beta>0$, 
confidence level $\delta\in(0,1)$, and 
integers $n\ge 1$ and $K\ge 1$, 
if $2\alpha_{n,K,\delta}<1$, $\tau^\star$ satisfies
\[
\begin{aligned}
&\Pr_{(X,Y)\sim D_\theta}\!\big(|\ell_\theta(X,Y)|>\tau^\star\big)
\;\le\;\\& \qquad \qquad
2\,\alpha_{n,K,\delta}
\;\le\;
\Pr_{(X,Y)\sim D_\theta}\!\big(|\ell_\theta(X,Y)|\ge\tau^\star\big),
\end{aligned}
\]
if, in addition, $\Pr_{(X,Y)\sim D_\theta}(|\ell_\theta(X,Y)|=\tau^\star)=0$, we have
\[
\Pr_{(X,Y)\sim D_\theta}\!\big(|\ell_\theta(X,Y)|>\tau^\star\big) \;=\; 2\,\alpha_{n,K,\delta},
\]
and, equivalently, $\tau^\star$ is the $(1-2\alpha_{n,K,\delta})$-quantile of $|\ell_\theta(X,Y)|$ under $D_\theta$. Otherwise, if $2\alpha_{n,K,\delta}\ge 1$, we have $\tau^\star=0$ is a minimizer of $\tau\mapsto B_\theta(\tau)$ over $\tau\ge 0$.
\end{corollary}

Detailed proofs are in Appendix~\ref{app:proofs:fixed}.



\begin{remark}
\label{rem:track_trunc}
Corollary~\ref{cor:opt_tau} suggests choosing $\tau$ so that the clipping fraction $\Pr(|\ell_\theta|>\tau)$ matches the target level $2\alpha_{n,K,\delta}$. 
As the evaluation budget ($n$ or $K$) increases, $\alpha_{n,K,\delta}$ decreases. Consequently, the target clipping fraction decreases and the recommended threshold $\tau$ increases. This quantile-based rule automatically relaxes clipping as Monte Carlo error diminishes.
\end{remark}


\paragraph{Threshold calibration}
Practitioners often treat the clipping threshold $\tau$ as a static hyperparameter that requires manual tuning.
Corollary~\ref{cor:opt_tau} instead yields a direct, budget-aware calibration rule.
Given an evaluation batch $\{(x_i,y_{i,j})\}_{i\le n,\,j\le K}$, compute the log-ratio magnitudes $u_{i,j}:=|\ell_\theta(x_i,y_{i,j})|$.
If $2\alpha_{n,K,\delta}\ge 1$, set $\widehat{\tau}:=0$.
Otherwise, set $\widehat{\tau}$ to the empirical $(1-2\alpha_{n,K,\delta})$-quantile of $\{u_{i,j}\}$.
Algorithmically, this theory-guided rule balances the bias--variance trade-off by clipping approximately the top $2\alpha_{n,K,\delta}$ fraction of extreme log-ratios in the batch, thereby reducing reliance on heuristic hyperparameter sweeps.

Theorems~1--2 treat $\tau$ as fixed; when $\tau$ is selected from the evaluation sample (e.g., by an empirical quantile rule), the resulting procedure should be viewed as a practical calibration heuristic unless additional uniformity or sample-splitting arguments are used.

\subsection{Budget allocation across prompts, rollouts, and preference data}
\label{subsec:implications:budget}

Given a fixed computational budget, practitioners often face an allocation trade-off among prompts, rollouts per prompt, and preference data. This subsection provides theoretically grounded guidelines for this budget distribution.



\subsubsection{Uniform-cost baseline}
Suppose rollouts share the same cost and the sampling budget is bounded by $nK\le B$ for some $B>0$. Substituting $n=B/K$ into the leading-order sampling terms of Lemma~\ref{lem:main:stat} reveals that the upper bound is minimised at $K^\star=1$. Therefore, under a uniform-cost model, the range-based concentration bound strongly favours allocating budget to prompt coverage rather than additional rollouts per prompt. A detailed derivation is given in Appendix~\ref{app:proofs:budget}.


\subsubsection{Prefill and decode cost model}
In LLM inference, sampling costs are typically asymmetric across prompts and rollouts \citep{pope2023transformer_inference}.
Evaluating a new prompt requires a forward pass over prompt tokens to construct an attention cache, whereas additional rollouts reuse this cache, primarily incurring incremental decoding costs \citep{kwon2023pagedattention}.
We model this asymmetry by separating a prefill and a decode cost, imposing the constraint:
\(
B \;\ge\; n\,c_{\mathrm{prefill}} \;+\; nK\,c_{\mathrm{decode}}
\).
Substituting $n=B/(c_{\mathrm{prefill}}+Kc_{\mathrm{decode}})$ into the dominant sampling structure isolates a one-dimensional objective in $K$.
Treating $K\ge 1$ as a continuous variable in the leading-order proxy from Lemma~\ref{lem:main:stat} yields the following optimal allocation rule.


\begin{corollary}[Optimal rollout allocation]
\label{cor:opt_K_hoeffding}
The continuous proxy minimiser over $K\ge 1$ satisfies
\[
K^\star \;=\; \max\!\left\{1,\left(\frac{c_{\mathrm{prefill}}}{c_{\mathrm{decode}}}\right)^{2/3}\right\}.
\]
In practice, one may take $K=\lfloor K^\star \rceil$, and then set $n$ by the budget constraint.
\end{corollary}

Detailed proofs are in Appendix~\ref{app:proofs:budget}.

\begin{remark}
The expression for $K^\star$ depends only on the ratio $c_{\mathrm{prefill}}/c_{\mathrm{decode}}$ because the shared range multiplier $1+2\beta\tau$ does not affect the minimiser.
The $2/3$ power law implies that the optimal number of rollouts per prompt grows sublinearly with
$c_{\mathrm{prefill}}/c_{\mathrm{decode}}$.
\end{remark}

\paragraph{Variance-aware refinement}
{The range-based sampling simplification above is conservative because it does not separate prompt-level variability from rollout-level variability. A refinement is to use a two-stage variance decomposition. Let $Z$ denote a per-rollout contribution in the empirical objective (see the variables $Z_{i,j}$ in the proof of Lemma~\ref{lem:main:rollout} in Appendix~\ref{app:proofs:stat}), and define}
\[
\sigma_{\mathrm{prompt}}^2:=\mathrm{Var}\!\left(\E[Z\mid X]\right),
\quad
\sigma_{\mathrm{rollout}}^2:=\E\!\left[\mathrm{Var}(Z\mid X)\right].\]

\begin{corollary}
\label{cor:opt_K_variance}
Under the same cost constraint $B \ge n\,c_{\mathrm{prefill}} + nK\,c_{\mathrm{decode}}$, optimising the resulting variance proxy yields an allocation rule of the form
\[
K^\star \;\approx\; \max\!\left\{1,\sqrt{\frac{c_{\mathrm{prefill}}}{c_{\mathrm{decode}}}\cdot
\frac{\sigma_{\mathrm{rollout}}^2}{\sigma_{\mathrm{prompt}}^2}}\right\}.
\]
\end{corollary}

A proof is given in Appendix~\ref{app:proofs:budget}.




\subsubsection{Preference data}
Beyond prompts and rollouts, the reward shift error introduces an additional budget consideration. By Lemma~\ref{lem:main:shift}, this term depends on the reward-model training error \(L_{\mathrm{train}}^{(2)}(\phi)\) and the coverage coefficient \(\Ccov(\theta)\). Preference data collection therefore affects the bound in two ways. Increasing relevant preference data can improve reward-model fit on the training distribution, and collecting data closer to the policy-induced distribution can reduce the mismatch captured by \(\Ccov(\theta)\). These observations provide guidance for preference data collection through their effect on the reward shift term, although the present analysis does not derive an explicit allocation rule in terms of the number of preference labels. This implication is most relevant when the sampling terms are no longer the dominant terms in the bound.

\section{Conclusions}

Alignment and adaptation in large language models (LLMs) are now driven by reinforcement learning from human feedback (RLHF), but a rigorous theory of how RLHF generalises is still underdeveloped, particularly when the reward could shift, and a KL clipping regularisation is implemented. To address this gap, we develop generalisation theory for RLHF that explicitly models two key practical effects: (1) distribution shift between the data used to train the reward model and the policy-induced distribution encountered at deployment, and (2) statistical noise introduced by empirical estimation of the clipped KL regulariser. We prove high-probability generalisation bounds that decompose the generalisation error into interpretable components, including sampling error from both prompts and rollouts, reward shift error, and KL clipping error. 
Our theory suggests optimal KL clipping threshold rules, quantitative budget allocation guidance on prompts and rollouts, and guidance for preference data collection through the reward shift term.



\bibliography{reference}

\newpage

\onecolumn

\title{Generalisation of RLHF under Reward Shift and Clipped KL Regularisation\\(Supplementary Material)}
\maketitle

\appendix

\section{Notation}


\newcolumntype{S}{>{\raggedright\arraybackslash}p{0.34\textwidth}} 
\newcolumntype{M}{>{\raggedright\arraybackslash}X}

\begin{table*}[h]
\centering
\caption{Notation}
\begin{tabularx}{\textwidth}{@{} l X @{}}
\toprule
\textbf{Symbol} & \textbf{Meaning} \\
\midrule
$\mathcal{X}, \mathcal{Y}$ & Prompt space and response space. \\
$(x,y)$ & A prompt-response pair. \\
$\rho$ & Prompt distribution used for post-training / evaluation. \\
$\rho_{\mathrm{label}}$ & Prompt distribution used for collecting preference data (reward modelling). \\
$\pi(\cdot\mid x)$ & A policy: conditional distribution over responses given prompt $x$. \\
$\pi_\theta$ & Post-trained policy, parameterised by $\theta$. \\
$\Theta$ & Policy parameter space. \\
$\pi_{\mathrm{ref}}$ & Reference policy (typically an SFT model). \\
$\theta$ & Parameters of the policy $\pi_\theta$. \\
$\Phi$ & Reward-model parameter space. \\
$\phi$ & Parameters of the learned reward model $\hat r_\phi$. \\
$r^\star:\mathcal{X}\times\mathcal{Y}\to[0,1]$ & Target (oracle) reward function. \\
$\hat r_\phi:\mathcal{X}\times\mathcal{Y}\to[0,1]$ & Learned reward model with parameters $\phi$. \\
$e_\phi(x,y)$ & Reward-model error, typically $e_\phi(x,y)=\hat r_\phi(x,y)-r^\star(x,y)$. \\
$D_{\mathrm{train}}$ & Joint distribution for reward-model training, e.g.\ $D_{\mathrm{train}}(x,y)=\rho_{\mathrm{label}}(x)\,\pi_{\mathrm{ref}}(y\mid x)$. \\
$D_\theta$ & Policy-induced joint distribution, $D_\theta(x,y)=\rho(x)\,\pi_\theta(y\mid x)$. \\
$L^{(2)}_{\mathrm{train}}(\phi)$ & Reward-model MSE on $D_{\mathrm{train}}$: $\mathbb{E}_{(X,Y)\sim D_{\mathrm{train}}}[e_\phi(X,Y)^2]$. \\
$\chi^2(D_\theta\|D_{\mathrm{train}})$ & Chi-square divergence measuring coverage / shift from $D_{\mathrm{train}}$ to $D_\theta$. \\
$C(\theta)$ & Coverage coefficient, typically $C(\theta)=\sqrt{1+\chi^2(D_\theta\|D_{\mathrm{train}})}$. \\
$C_{\mathrm{prompt}}$ & Prompt-shift component of coverage (in a factorisation of $C(\theta)$). \\
$C_{\mathrm{pol}}(\theta)$ & Policy-shift component of coverage (in a factorisation of $C(\theta)$). \\
$\beta>0$ & KL-regularisation strength (penalty coefficient). \\
$\ell_\theta(x,y)$ & Log-ratio, $\ell_\theta(x,y)=\log \pi_\theta(y\mid x)-\log \pi_{\mathrm{ref}}(y\mid x)$. \\
$\tau>0$ & Clipping threshold for log-ratios. \\
$\ell_\theta^\tau(x,y)$ & Clipped log-ratio, $\ell_\theta^\tau(x,y)=\mathrm{clip}(\ell_\theta(x,y),-\tau,\tau)$. \\
$\mathrm{KL}(\pi_\theta(\cdot\mid x)\|\pi_{\mathrm{ref}}(\cdot\mid x))$ & Reference KL at prompt $x$ (population expectation of $\ell_\theta(x,Y)$ under $Y\sim\pi_\theta(\cdot\mid x)$). \\
$J_{r}(\theta)$ & Population objective under reward $r$: $\mathbb{E}_{X\sim\rho,Y\sim\pi_\theta(\cdot\mid X)}[r(X,Y)-\beta\,\ell_\theta(X,Y)]$. \\
$J_{r,\tau}(\theta)$ & Clipped population objective: replace $\ell_\theta$ by $\ell_\theta^\tau$ in $J_r(\theta)$. \\
$J^\star(\theta)$ & Target objective, typically $J^\star(\theta)=J_{r^\star}(\theta)$. \\
$J^\phi(\theta)$ & Learned-reward objective, typically $J^\phi(\theta)=J_{\hat r_\phi}(\theta)$. \\
$J^{\phi,\tau}(\theta)$ & Learned-reward clipped objective, typically $J^{\phi,\tau}(\theta)=J_{\hat r_\phi,\tau}(\theta)$. \\
\bottomrule
\end{tabularx}
\label{tab:notation-core}
\end{table*}

\begin{table*}[t]
\centering
\begin{tabularx}{\textwidth}{@{} l X @{}}
\toprule
\textbf{Symbol} & \textbf{Meaning} \\
\midrule
$\widehat J^{r,\tau}_{n,K}(\theta)$ & Empirical objective using $n$ prompts and $K$ rollouts per prompt (reward $r$, clipping $\tau$). \\
$\widehat J^{r,\tau}_{n,\infty}(\theta)$ & Conditional (infinite-rollout) analogue: expectation over rollouts given the $n$ sampled prompts. \\
$n$ & Number of sampled prompts. \\
$K$ & Number of rollouts per prompt. \\
$P$ & Prior distribution over $\Theta$ (PAC-Bayes). \\
$Q$ & Posterior distribution over $\Theta$ (PAC-Bayes). \\
$\mathrm{KL}(Q\|P)$ & PAC-Bayes complexity term. \\
$\delta\in(0,1)$ & Confidence parameter for high-probability bounds. \\
\bottomrule
\end{tabularx}
\end{table*}
\section{Definitions and Lemmas}
\label{app:tools}

\begin{definition}[$\KL$ divergence {\citep{kullback1951}}]
\label{def:app:kl}
Suppose that $P$ is absolutely continuous with respect to $Q$.
The $\KL$ divergence is defined by
\[
\KL(P\|Q)
:=
\int p(x) \log \!\left( \frac{p(x)}{q(x)} \right) dx.
\]
\end{definition}

\begin{lemma}[Hoeffding's inequality {\citep{hoeffding1963probability}}]
\label{lem:app:hoeffding}
Let $Z_1,\dots,Z_N$ be independent random variables.
Assume there exist constants $a\le b$ such that $a\le Z_i\le b$ almost surely for every $i$.
Then, for any $\delta\in(0,1)$, with probability at least $1-\delta$,
\[
\abs{
\frac{1}{N}\sum_{i=1}^{N} Z_i
-
\E\!\left[\frac{1}{N}\sum_{i=1}^{N} Z_i\right]
}
\le
(b-a)\sqrt{\frac{\log(2/\delta)}{2N}}.
\]
\end{lemma}

\begin{lemma}[Hoeffding's lemma {\citep{boucheron2013concentration}}]
\label{lem:app:hoeffding_lemma}
Let $Z$ be a random variable and assume $a\le Z\le b$ almost surely.
Then, for any $\lambda\in\mathbb{R}$,
\[
\E\!\left[\exp\bigl(\lambda(Z-\E[Z])\bigr)\right]
\le
\exp\!\left(\frac{\lambda^2(b-a)^2}{8}\right).
\]
\end{lemma}

\begin{lemma}[Change of measure {\citep{catoni2007pac}}]
\label{lem:app:com}
Let $P$ and $Q$ be distributions on $\Theta$ such that $\KL(Q\|P)<\infty$.
Let $F:\Theta\to\mathbb{R}$ satisfy $\E_{\theta\sim P}[\exp(F(\theta))]<\infty$.
Then,
\[
\E_{\theta\sim Q}[F(\theta)]
\le
\KL(Q\|P)
+
\log \E_{\theta\sim P}\!\left[\exp(F(\theta))\right].
\]
\end{lemma}

\begin{proof}
Let $p$ and $q$ denote densities of $P$ and $Q$ with respect to a common reference.
By definition, $\KL(Q\|P)=\E_Q[\log(q/p)]$.

Start from the identity
\[
\E_Q[F]
=
\E_Q[\log(e^F)].
\]
Insert the density ratio $p/q$ inside the logarithm:
\[
\E_Q[F]
=
\E_Q\!\left[\log\!\left(e^F\frac{p}{q}\right)\right]
+
\E_Q\!\left[\log\!\left(\frac{q}{p}\right)\right].
\]
The second term is exactly $\KL(Q\|P)$.
For the first term, Jensen's inequality gives
\[
\E_Q\!\left[\log\!\left(e^F\frac{p}{q}\right)\right]
\le
\log \E_Q\!\left[e^F\frac{p}{q}\right]
=
\log \E_P[e^F].
\]
Substituting these two relations into the previous display yields
\[
\E_Q[F]
\le
\KL(Q\|P)+\log \E_P[e^F],
\]
which is the claimed inequality.
\end{proof}

\begin{definition}[$\chisq$ divergence {\citep{tsybakov2009}}]
\label{def:app:chisq}
Suppose that $D_\theta$ is absolutely continuous with respect to $D_{\mathrm{train}}$.
The $\chisq$ divergence is defined by
\[
\chisq(D_\theta\|D_{\mathrm{train}})
:=
\E_{(X,Y)\sim D_{\mathrm{train}}}\!\left[\left(\frac{D_\theta(X,Y)}{D_{\mathrm{train}}(X,Y)}-1\right)^2\right].
\]
\end{definition}

\begin{lemma}[$\chisq$ change of measure]
\label{lem:app:chisq_com}
Let $P$ and $Q$ be distributions on a common space and assume $Q\ll P$.
Let $w=\frac{dQ}{dP}$ and assume $\chisq(Q\|P)<\infty$.
If $f$ satisfies $\E_{Z\sim P}[f(Z)^2]<\infty$, we have
\[
\abs{\E_{Z\sim Q}[f(Z)]}
\le
\sqrt{1+\chisq(Q\|P)}\,
\sqrt{\E_{Z\sim P}[f(Z)^2]}.
\]
\end{lemma}

\begin{proof}
Because $Q\ll P$, the density ratio $w=\frac{dQ}{dP}$ exists and the expectation under $Q$ can be written as
\[
\E_{Z\sim Q}[f(Z)]
=
\E_{Z\sim P}[w(Z)f(Z)].
\]
Applying Cauchy--Schwarz to the right-hand side gives
\[
\abs{\E_{P}[wf]}
\le
\sqrt{\E_{P}[w^2]}\,\sqrt{\E_{P}[f^2]}.
\]
It remains to express $\E_P[w^2]$ in terms of $\chisq(Q\|P)$.
By definition,
\[
\chisq(Q\|P)
=
\E_P[(w-1)^2]
=
\E_P[w^2]-2\E_P[w]+1.
\]
Also $\E_P[w]=1$, since $w=dQ/dP$ integrates to $1$ under $P$.
Substituting $\E_P[w]=1$ into the previous identity yields $\E_P[w^2]=1+\chisq(Q\|P)$.
Plugging this into the Cauchy--Schwarz bound gives
\[
\abs{\E_{Z\sim Q}[f(Z)]}
\le
\sqrt{1+\chisq(Q\|P)}\,\sqrt{\E_{Z\sim P}[f(Z)^2]},
\]
which completes the proof.
\end{proof}

\begin{lemma}[Monte Carlo estimation of the clipped log ratio]
\label{lem:main:kl_mc}
Under the same conditions of Lemma~\ref{lem:main:rollout}, with probability at least $1-\delta$ over the evaluation prompts and rollouts,
\begin{equation}
\label{eq:main:kl_mc}
\begin{aligned}
\left|
\widehat \kappa_{n,K}^{\tau}(\theta)-\kappa^\tau(\theta)
\right|
\le
2\tau
\left(
\sqrt{\frac{\log(4/\delta)}{2n}}
+
\sqrt{\frac{\log(4/\delta)}{2nK}}
\right).
\end{aligned}
\end{equation}
\end{lemma}

\begin{lemma}[KL divergence between Gaussian distributions {\citep{murphy2022pml1}}]
\label{lem:app:ou_gauss_kl}
Let $Q=\mathcal{N}(\mu_Q,\Sigma_Q)$ and $P=\mathcal{N}(\mu_P,\Sigma_P)$ be Gaussian distributions on $\mathbb{R}^d$, where $\Sigma_Q\succ 0$ and $\Sigma_P\succ 0$.
Then,
\begin{equation}
\label{eq:app:ou_gauss_kl}
\begin{aligned}
\KL(Q\|P)
=
\frac12\Big(
\tr(\Sigma_P^{-1}\Sigma_Q)
+
(\mu_Q-\mu_P)^\top\Sigma_P^{-1}(\mu_Q-\mu_P)
-d
+
\log\frac{\det(\Sigma_P)}{\det(\Sigma_Q)}
\Big).
\end{aligned}
\end{equation}
\end{lemma}
\section{Proofs}
\label{app:proofs}

\subsection{Error decomposition}
\label{app:proofs:decomp}

\begin{proof}[Proof of Lemma~\ref{prop:main:decomp}]
Let $\theta\in\Theta$ and $\phi\in\Phi$ be arbitrary, and let $\tau>0$ be an arbitrary clipping threshold.
The argument is a purely algebraic decomposition in which two intermediate population objectives are inserted between the empirical surrogate objective and the target objective.

Consider the difference $\widehat J_{n,K}^{\phi,\tau}(\theta)-J^\star(\theta)$.
Add and subtract the intermediate quantities $J^{\phi,\tau}(\theta)$ and $J^\phi(\theta)$ to obtain
\begin{align*}
\widehat J_{n,K}^{\phi,\tau}(\theta)-J^\star(\theta)
={}&
\widehat J_{n,K}^{\phi,\tau}(\theta)-J^{\phi,\tau}(\theta)
+
J^{\phi,\tau}(\theta)-J^\phi(\theta)
+
J^\phi(\theta)-J^\star(\theta).
\end{align*}
Taking absolute values and applying the triangle inequality gives
\begin{align*}
\abs{\widehat J_{n,K}^{\phi,\tau}(\theta)-J^\star(\theta)}
\le{}&
\abs{\widehat J_{n,K}^{\phi,\tau}(\theta)-J^{\phi,\tau}(\theta)}
+
\abs{J^{\phi,\tau}(\theta)-J^\phi(\theta)}
+
\abs{J^\phi(\theta)-J^\star(\theta)}.
\end{align*}
This is exactly the inequality stated in Lemma~\ref{prop:main:decomp}.
\end{proof}

\subsection{Statistical error}
\label{app:proofs:stat}

\begin{proof}[Proof of Lemma~\ref{lem:main:rollout}]
Let $\theta\in\Theta$ be an arbitrary policy parameter.
Let $r:\cX\times\cY\to[0,1]$ be an arbitrary reward function, let $\tau>0$ be an arbitrary clipping threshold, and let $\delta\in(0,1)$ be an arbitrary confidence level.

The goal is to control the Monte Carlo deviation arising from drawing only $K$ rollouts per prompt, while conditioning on the realized prompts.
Let $x_1,\dots,x_n$ denote the realized prompts.
For each $i\in\{1,\dots,n\}$ and each rollout index $j\in\{1,\dots,K\}$, define the per-rollout contribution
\[
Z_{i,j}
:=
r(x_i,y_{i,j})
-
\beta\,\ell_\theta^\tau(x_i,y_{i,j}).
\]
By the definition of the empirical objective, one can rewrite
\[
\widehat J_{n,K}^{r,\tau}(\theta)
=
\frac{1}{nK}\sum_{i=1}^n\sum_{j=1}^K Z_{i,j}.
\]

Next define the conditional expectation of the empirical objective given the prompts.
For each fixed prompt $x_i$, conditional on $x_i$ the rollout $y_{i,j}$ is distributed as $\pi_\theta(\cdot\mid x_i)$, hence
\[
\E[Z_{i,j}\mid x_i]
=
\E_{Y\sim\pi_\theta(\cdot\mid x_i)}[r(x_i,Y)]
-
\beta\,\E_{Y\sim\pi_\theta(\cdot\mid x_i)}[\ell_\theta^\tau(x_i,Y)].
\]
Averaging these conditional expectations over $i$ yields the infinite-rollout analogue
\[
\widehat J_{n,\infty}^{r,\tau}(\theta)
:=
\frac{1}{n}\sum_{i=1}^n \E[Z_{i,1}\mid x_i].
\]
By construction,
\[
\E\!\left[\widehat J_{n,K}^{r,\tau}(\theta)\mid x_{1:n}\right]
=
\widehat J_{n,\infty}^{r,\tau}(\theta).
\]

To apply Hoeffding's inequality, it remains to verify a uniform bound on each $Z_{i,j}$.
Because $r(x_i,y_{i,j})\in[0,1]$ and $\ell_\theta^\tau(x_i,y_{i,j})\in[-\tau,\tau]$, it follows that
\[
-\beta\tau
\le
Z_{i,j}
\le
1+\beta\tau,
\]
so the interval width is $1+2\beta\tau$.

Conditional on the prompts $x_{1:n}$, the rollouts are independent across all index pairs $(i,j)$.
Therefore the collection $\{Z_{i,j}\}_{i\le n,\,j\le K}$ is independent conditional on $x_{1:n}$.
Applying Lemma~\ref{lem:app:hoeffding} to the average of these $nK$ bounded independent random variables, with failure probability $\delta$, gives that with probability at least $1-\delta$ over the rollouts conditional on $x_{1:n}$,
\[
\abs{
\widehat J_{n,K}^{r,\tau}(\theta)
-
\E\!\left[\widehat J_{n,K}^{r,\tau}(\theta)\mid x_{1:n}\right]
}
\le
(1+2\beta\tau)\sqrt{\frac{\log(2/\delta)}{2nK}}.
\]
Replacing the conditional expectation by $\widehat J_{n,\infty}^{r,\tau}(\theta)$ yields
\[
\abs{\widehat J_{n,K}^{r,\tau}(\theta)-\widehat J_{n,\infty}^{r,\tau}(\theta)}
\le
(1+2\beta\tau)\sqrt{\frac{\log(2/\delta)}{2nK}},
\]
which is the conclusion of Lemma~\ref{lem:main:rollout}.
\end{proof}

\begin{proof}[Proof of Lemma~\ref{lem:main:prompt}]
Let $\theta\in\Theta$ be an arbitrary policy parameter.
Let $r:\cX\times\cY\to[0,1]$ be an arbitrary reward function, let $\tau>0$ be an arbitrary clipping threshold, and let $\delta\in(0,1)$ be an arbitrary confidence level.

This lemma controls the deviation due only to sampling finitely many prompts, after taking the conditional expectation over rollouts.
Define, for each prompt $x\in\cX$,
\[
g_\theta^{r,\tau}(x)
=
\E_{Y\sim\pi_\theta(\cdot\mid x)}[r(x,Y)]
-
\beta\,\E_{Y\sim\pi_\theta(\cdot\mid x)}[\ell_\theta^\tau(x,Y)].
\]
Because $r(\cdot,\cdot)\in[0,1]$ and $\ell_\theta^\tau(\cdot,\cdot)\in[-\tau,\tau]$ pointwise, the first expectation lies in $[0,1]$ and the second expectation lies in $[-\tau,\tau]$.
Consequently, for every $x$,
\[
-\beta\tau
\le
g_\theta^{r,\tau}(x)
\le
1+\beta\tau,
\]
so the interval width is again $1+2\beta\tau$.

By definition,
\[
\widehat J_{n,\infty}^{r,\tau}(\theta)
=
\frac{1}{n}\sum_{i=1}^n g_\theta^{r,\tau}(x_i),
\qquad
J^{r,\tau}(\theta)
=
\E_{X\sim\rho}[g_\theta^{r,\tau}(X)].
\]
Since $x_1,\dots,x_n$ are independent draws from $\rho$, the sequence $g_\theta^{r,\tau}(x_1),\dots,g_\theta^{r,\tau}(x_n)$ consists of i.i.d.\ random variables bounded in an interval of width $1+2\beta\tau$.
Applying Lemma~\ref{lem:app:hoeffding} with $N=n$ and failure probability $\delta$ yields that with probability at least $1-\delta$ over the prompts,
\[
\abs{\widehat J_{n,\infty}^{r,\tau}(\theta)-J^{r,\tau}(\theta)}
\le
(1+2\beta\tau)\sqrt{\frac{\log(2/\delta)}{2n}}.
\]
This is precisely the statement of Lemma~\ref{lem:main:prompt}.
\end{proof}

\begin{proof}[Proof of Lemma~\ref{lem:main:stat}]
Let $\theta\in\Theta$ be an arbitrary policy parameter.
Let $r:\cX\times\cY\to[0,1]$ be an arbitrary reward function, let $\tau>0$ be an arbitrary clipping threshold, and let $\delta\in(0,1)$ be an arbitrary confidence level.

The proof combines the two previous concentration statements by enforcing that they hold on a common high-probability event, and then applying a triangle inequality.

Define the rollout concentration event
\[
\mathcal E_{\mathrm{roll}}
:=
\left\{
\abs{\widehat J_{n,K}^{r,\tau}(\theta)-\widehat J_{n,\infty}^{r,\tau}(\theta)}
\le
(1+2\beta\tau)\sqrt{\frac{\log(4/\delta)}{2nK}}
\right\}.
\]
Lemma~\ref{lem:main:rollout} applied with confidence parameter $\delta/2$ implies that, conditional on $x_{1:n}$,
\[
\Pr(\mathcal E_{\mathrm{roll}}\mid x_{1:n})
\ge
1-\delta/2.
\]

Define the prompt concentration event
\[
\mathcal E_{\mathrm{prompt}}
:=
\left\{
\abs{\widehat J_{n,\infty}^{r,\tau}(\theta)-J^{r,\tau}(\theta)}
\le
(1+2\beta\tau)\sqrt{\frac{\log(4/\delta)}{2n}}
\right\}.
\]
Lemma~\ref{lem:main:prompt} applied with confidence parameter $\delta/2$ yields
\[
\Pr(\mathcal E_{\mathrm{prompt}})
\ge
1-\delta/2.
\]

Let $\mathcal E_{\mathrm{stat}}:=\mathcal E_{\mathrm{roll}}\cap\mathcal E_{\mathrm{prompt}}$.
By the union bound,
\[
\Pr(\mathcal E_{\mathrm{stat}})
\ge
1-\delta.
\]

Assume that $\mathcal E_{\mathrm{stat}}$ holds.
Then, the triangle inequality gives
\begin{align*}
\abs{\widehat J_{n,K}^{r,\tau}(\theta)-J^{r,\tau}(\theta)}
\le{}&
\abs{\widehat J_{n,K}^{r,\tau}(\theta)-\widehat J_{n,\infty}^{r,\tau}(\theta)}
+
\abs{\widehat J_{n,\infty}^{r,\tau}(\theta)-J^{r,\tau}(\theta)} \\
\le{}&
(1+2\beta\tau)
\left(
\sqrt{\frac{\log(4/\delta)}{2n}}
+
\sqrt{\frac{\log(4/\delta)}{2nK}}
\right),
\end{align*}
which is exactly the inequality claimed in Lemma~\ref{lem:main:stat}.
\end{proof}

\begin{proof}[Proof of Lemma~\ref{lem:main:kl_mc}]
Let $\theta\in\Theta$ be an arbitrary policy parameter, let $\tau>0$ be an arbitrary clipping threshold, and let $\delta\in(0,1)$ be an arbitrary confidence level.
Recall that $x_1,\dots,x_n$ are independent draws from $\rho$, and that, conditional on each $x_i$, the rollouts $y_{i,1},\dots,y_{i,K}$ are independent draws from $\pi_\theta(\cdot\mid x_i)$.
Define the per-rollout clipped log ratio
\[
Z_{i,j}
:=
\ell_\theta^\tau(x_i,y_{i,j}),
\]
so that, by the definition of $\widehat \kappa_{n,K}^{\tau}(\theta)$,
\[
\widehat \kappa_{n,K}^{\tau}(\theta)
=
\frac{1}{nK}\sum_{i=1}^n\sum_{j=1}^K Z_{i,j}.
\]
Because $\ell_\theta^\tau(x,y)=\clip(\ell_\theta(x,y),-\tau,\tau)$ by definition, it follows that $Z_{i,j}\in[-\tau,\tau]$ almost surely for all $(i,j)$, and therefore each $Z_{i,j}$ is bounded in an interval of width $2\tau$.

To make the two-stage sampling structure explicit, introduce the conditional infinite-rollout analogue
\[
\widehat \kappa_{n,\infty}^{\tau}(\theta)
:=
\frac{1}{n}\sum_{i=1}^n \E\!\left[ Z_{i,1}\mid x_i \right]
=
\frac{1}{n}\sum_{i=1}^n \E_{Y\sim\pi_\theta(\cdot\mid x_i)}\!\left[\ell_\theta^\tau(x_i,Y)\right].
\]
By construction, conditional on the realized prompts $x_{1:n}$, the random variables $\{Z_{i,j}\}_{i\le n,\,j\le K}$ are independent, and moreover
\[
\E\!\left[\widehat \kappa_{n,K}^{\tau}(\theta)\mid x_{1:n}\right]
=
\widehat \kappa_{n,\infty}^{\tau}(\theta).
\]
Applying Lemma~\ref{lem:app:hoeffding} to the average of the $nK$ bounded independent random variables $\{Z_{i,j}\}$, conditional on $x_{1:n}$ and with failure probability $\delta/2$, yields that with probability at least $1-\delta/2$ over the rollouts conditional on $x_{1:n}$,
\[
\left|
\widehat \kappa_{n,K}^{\tau}(\theta)-\widehat \kappa_{n,\infty}^{\tau}(\theta)
\right|
\le
2\tau\sqrt{\frac{\log(4/\delta)}{2nK}}.
\]

It remains to control the deviation due to sampling only finitely many prompts.
Define the prompt-level functional
\[
h_\theta^\tau(x)
:=
\E_{Y\sim\pi_\theta(\cdot\mid x)}\!\left[\ell_\theta^\tau(x,Y)\right].
\]
Since $\ell_\theta^\tau(x,Y)\in[-\tau,\tau]$ almost surely under $Y\sim\pi_\theta(\cdot\mid x)$, it follows that $h_\theta^\tau(x)\in[-\tau,\tau]$ for every $x$, and thus $h_\theta^\tau(X)$ is bounded in an interval of width $2\tau$ when $X\sim\rho$.
By the definition of $\widehat \kappa_{n,\infty}^{\tau}(\theta)$,
\[
\widehat \kappa_{n,\infty}^{\tau}(\theta)
=
\frac{1}{n}\sum_{i=1}^n h_\theta^\tau(x_i).
\]
Moreover, by the definition of $D_\theta(x,y)=\rho(x)\pi_\theta(y\mid x)$, the clipped population average can be written as
\[
\kappa^\tau(\theta)
=
\E_{(X,Y)\sim D_\theta}\!\left[\ell_\theta^\tau(X,Y)\right]
=
\E_{X\sim\rho}\!\left[h_\theta^\tau(X)\right].
\]
Since $x_1,\dots,x_n$ are independent draws from $\rho$, the sequence $h_\theta^\tau(x_1),\dots,h_\theta^\tau(x_n)$ consists of i.i.d.\ random variables bounded in an interval of width $2\tau$.
Applying Lemma~\ref{lem:app:hoeffding} with $N=n$ and failure probability $\delta/2$ yields that with probability at least $1-\delta/2$ over the prompts,
\[
\left|
\widehat \kappa_{n,\infty}^{\tau}(\theta)-\kappa^\tau(\theta)
\right|
\le
2\tau\sqrt{\frac{\log(4/\delta)}{2n}}.
\]

Finally, consider the event on which both of the preceding inequalities hold.
By the union bound, this event has probability at least $1-\delta$ over the joint draw of prompts and rollouts.
On this event, the triangle inequality implies
\begin{align*}
\left|
\widehat \kappa_{n,K}^{\tau}(\theta)-\kappa^\tau(\theta)
\right|
&\le
\left|
\widehat \kappa_{n,K}^{\tau}(\theta)-\widehat \kappa_{n,\infty}^{\tau}(\theta)
\right|
+
\left|
\widehat \kappa_{n,\infty}^{\tau}(\theta)-\kappa^\tau(\theta)
\right| \\
&\le
2\tau
\left(
\sqrt{\frac{\log(4/\delta)}{2n}}
+
\sqrt{\frac{\log(4/\delta)}{2nK}}
\right),
\end{align*}
which is exactly the claimed bound in \eqref{eq:main:kl_mc}.
\end{proof}

\subsection{Reward shift and surrogate bias}
\label{app:proofs:shiftbias}

\begin{proof}[Proof of Lemma~\ref{lem:main:shift}]
Let $\theta\in\Theta$ and $\phi\in\Phi$ be arbitrary parameters.
The proof begins by expressing the objective gap as an expectation of reward-model error under the deployment distribution, and then transferring this expectation back to the reward-model training distribution via a density ratio.

By definition,
\[
J^\phi(\theta)
=
\E_{X\sim\rho}\E_{Y\sim\pi_\theta(\cdot\mid X)}[\hat r_\phi(X,Y)]
-
\beta\,\E_{X\sim\rho}\KL(\pi_\theta(\cdot\mid X)\|\piref(\cdot\mid X)),
\]
and
\[
J^\star(\theta)
=
\E_{X\sim\rho}\E_{Y\sim\pi_\theta(\cdot\mid X)}[r^\star(X,Y)]
-
\beta\,\E_{X\sim\rho}\KL(\pi_\theta(\cdot\mid X)\|\piref(\cdot\mid X)).
\]
The KL regularization terms coincide, so they cancel after subtraction, giving
\[
J^\phi(\theta)-J^\star(\theta)
=
\E_{X\sim\rho}\E_{Y\sim\pi_\theta(\cdot\mid X)}[\hat r_\phi(X,Y)-r^\star(X,Y)].
\]
Introduce the pointwise reward-model error $e_\phi(x,y)=\hat r_\phi(x,y)-r^\star(x,y)$.
Using the joint distribution $\Dtheta(x,y)=\rho(x)\pi_\theta(y\mid x)$, the preceding display can be rewritten as
\[
J^\phi(\theta)-J^\star(\theta)
=
\E_{(X,Y)\sim\Dtheta}[e_\phi(X,Y)].
\]

Assume that $\Dtheta\ll\Drm$ and define the density ratio
\[
w_\theta(x,y)
:=
\frac{\Dtheta(x,y)}{\Drm(x,y)}.
\]
Then, the expectation under $\Dtheta$ can be written under $\Drm$ as
\[
\E_{(X,Y)\sim\Dtheta}[e_\phi(X,Y)]
=
\E_{(X,Y)\sim\Drm}[w_\theta(X,Y)e_\phi(X,Y)].
\]
Applying Cauchy--Schwarz yields
\[
\abs{\E_{\Drm}[w_\theta e_\phi]}
\le
\sqrt{\E_{\Drm}[w_\theta^2]}\,
\sqrt{\E_{\Drm}[e_\phi^2]}.
\]

The second factor is exactly $\sqrt{L_{\mathrm{train}}^{(2)}(\phi)}$ by the definition of $L_{\mathrm{train}}^{(2)}(\phi)$.
For the first factor, note that $\E_{\Drm}[w_\theta]=1$ and
\[
\chisq(\Dtheta\|\Drm)
=
\E_{\Drm}\bigl[(w_\theta-1)^2\bigr]
=
\E_{\Drm}[w_\theta^2]-1.
\]
Consequently, $\E_{\Drm}[w_\theta^2]=1+\chisq(\Dtheta\|\Drm)$.
Substituting these identities into the Cauchy--Schwarz bound gives
\[
\abs{J^\phi(\theta)-J^\star(\theta)}
\le
\sqrt{1+\chisq(\Dtheta\|\Drm)}\,
\sqrt{L_{\mathrm{train}}^{(2)}(\phi)}.
\]
By the definition of $\Ccov(\theta)$ in eq. \eqref{eq:pre:ccov}, this is
\[
\abs{J^\phi(\theta)-J^\star(\theta)}
\le
\Ccov(\theta)\,\sqrt{L_{\mathrm{train}}^{(2)}(\phi)},
\]
which is the statement of Lemma~\ref{lem:main:shift}.
\end{proof}

\begin{proof}[Proof of Lemma~\ref{lem:main:cov_factor}]
Let $\theta\in\Theta$ be arbitrary.
Assume that $\rho\ll\rholabel$ and that $\pi_\theta(\cdot\mid x)\ll\piref(\cdot\mid x)$ for every $x$ with $\rholabel(x)>0$.
Under these conditions, $\Dtheta\ll\Drm$ holds and the density ratio
\[
w_\theta(x,y)
:=
\frac{\Dtheta(x,y)}{\Drm(x,y)}
\]
is well defined on the support of $\Drm$.

By definition,
\[
\Ccov(\theta)^2
=
1+\chisq(\Dtheta\|\Drm)
=
\E_{(X,Y)\sim\Drm}\bigl[w_\theta(X,Y)^2\bigr].
\]
Using $\Drm(x,y)=\rholabel(x)\piref(y\mid x)$ and $\Dtheta(x,y)=\rho(x)\pi_\theta(y\mid x)$, one obtains the factorization
\[
w_\theta(x,y)
=
\frac{\rho(x)}{\rholabel(x)}
\cdot
\frac{\pi_\theta(y\mid x)}{\piref(y\mid x)}.
\]
Substituting this expression into the definition of $\Ccov(\theta)^2$ and taking expectation under $\Drm$ yields
\[
\Ccov(\theta)^2
=
\E_{X\sim\rholabel}
\left[
\left(\frac{\rho(X)}{\rholabel(X)}\right)^2
\E_{Y\sim\piref(\cdot\mid X)}
\left[
\left(\frac{\pi_\theta(Y\mid X)}{\piref(Y\mid X)}\right)^2
\right]
\right].
\]
By the definition of $\Cpol(\theta)$, the inner expectation is bounded above by $\Cpol(\theta)^2$ for each $x$ in the support of $\rholabel$.
Therefore,
\[
\Ccov(\theta)^2
\le
\Cpol(\theta)^2
\E_{X\sim\rholabel}
\left[
\left(\frac{\rho(X)}{\rholabel(X)}\right)^2
\right]
=
\Cpol(\theta)^2\,\Cprompt^2.
\]
Taking square roots yields $\Ccov(\theta)\le \Cprompt\,\Cpol(\theta)$.
\end{proof}

\begin{proof}[Proof of Lemma~\ref{lem:main:bias}]
Let $\theta\in\Theta$ and $\phi\in\Phi$ be arbitrary parameters, and let $\tau>0$ be an arbitrary clipping threshold.
The argument is an identity at the level of population objectives, followed by a standard absolute-value bound.

By definition of the clipped objective,
\[
J^{\phi,\tau}(\theta)
=
\E_{X\sim\rho}\E_{Y\sim\pi_\theta(\cdot\mid X)}[\hat r_\phi(X,Y)]
-
\beta\,\E_{X\sim\rho}\E_{Y\sim\pi_\theta(\cdot\mid X)}[\ell_\theta^\tau(X,Y)].
\]
Using $\Dtheta(x,y)=\rho(x)\pi_\theta(y\mid x)$, this can be written as
\[
J^{\phi,\tau}(\theta)
=
\E_{(X,Y)\sim\Dtheta}[\hat r_\phi(X,Y)]
-
\beta\,\E_{(X,Y)\sim\Dtheta}[\ell_\theta^\tau(X,Y)].
\]

For the exact objective, recall that
\[
\KL(\pi_\theta(\cdot\mid x)\|\piref(\cdot\mid x))
=
\E_{Y\sim\pi_\theta(\cdot\mid x)}[\ell_\theta(x,Y)].
\]
Substituting this identity into the definition of $J^\phi(\theta)$ yields
\[
J^\phi(\theta)
=
\E_{(X,Y)\sim\Dtheta}[\hat r_\phi(X,Y)]
-
\beta\,\E_{(X,Y)\sim\Dtheta}[\ell_\theta(X,Y)].
\]

Subtracting the two displays gives the exact identity
\[
J^{\phi,\tau}(\theta)-J^\phi(\theta)
=
\beta\,\E_{(X,Y)\sim\Dtheta}\bigl[\ell_\theta(X,Y)-\ell_\theta^\tau(X,Y)\bigr].
\]
Taking absolute values and using $\abs{\E[U]}\le \E[\abs{U}]$ yields
\[
\abs{J^{\phi,\tau}(\theta)-J^\phi(\theta)}
\le
\beta\,\E_{(X,Y)\sim\Dtheta}\bigl[\abs{\ell_\theta(X,Y)-\ell_\theta^\tau(X,Y)}\bigr],
\]
which is precisely the inequality asserted in Lemma~\ref{lem:main:bias}.
\end{proof}

\subsection{Unified fixed-policy bound}
\label{app:proofs:fixed}

\begin{proof}[Proof of Theorem~\ref{thm:main:fixed}]
Let $\theta\in\Theta$ and $\phi\in\Phi$ be arbitrary, and let $\tau>0$ and $\delta\in(0,1)$ be arbitrary.
Assume the conditions stated in Theorem~\ref{thm:main:fixed}, so that Lemmas~\ref{lem:main:stat}, \ref{lem:main:shift}, and \ref{lem:main:bias} are applicable.

Lemma~\ref{prop:main:decomp} provides the deterministic decomposition
\begin{align*}
\abs{\widehat J_{n,K}^{\phi,\tau}(\theta)-J^\star(\theta)}
\le{}&
\abs{\widehat J_{n,K}^{\phi,\tau}(\theta)-J^{\phi,\tau}(\theta)}
+
\abs{J^{\phi,\tau}(\theta)-J^\phi(\theta)}
+
\abs{J^\phi(\theta)-J^\star(\theta)}.
\end{align*}

To control the first term, apply Lemma~\ref{lem:main:stat} with $r=\hat r_\phi$.
With probability at least $1-\delta$ over the evaluation prompts and rollouts,
\[
\abs{\widehat J_{n,K}^{\phi,\tau}(\theta)-J^{\phi,\tau}(\theta)}
\le
(1+2\beta\tau)
\left(
\sqrt{\frac{\log(4/\delta)}{2n}}
+
\sqrt{\frac{\log(4/\delta)}{2nK}}
\right).
\]

The remaining two terms are controlled deterministically.
Lemma~\ref{lem:main:bias} gives
\[
\abs{J^{\phi,\tau}(\theta)-J^\phi(\theta)}
\le
\beta\,
\E_{(X,Y)\sim \Dtheta}
\!\left[
\abs{\ell_\theta(X,Y)-\ell_\theta^\tau(X,Y)}
\right],
\]
and Lemma~\ref{lem:main:shift} gives
\[
\abs{J^\phi(\theta)-J^\star(\theta)}
\le
\Ccov(\theta)\,\sqrt{L_{\mathrm{train}}^{(2)}(\phi)}.
\]

Substituting these three bounds into the decomposition yields the inequality stated in Theorem~\ref{thm:main:fixed}.
\end{proof}

\begin{proof}[Proof of Corollary~\ref{cor:opt_tau}]
Let $\theta\in\Theta$ be an arbitrary policy parameter, let $\phi\in\Phi$ be an arbitrary reward-model parameter, let $\beta>0$ be an arbitrary regularization coefficient, let $\delta\in(0,1)$ be an arbitrary confidence level, and let $n\ge 1$ and $K\ge 1$ be arbitrary integers.
Define
\[
\alpha_{n,K,\delta}
:=
\sqrt{\frac{\log(4/\delta)}{2n}}+\sqrt{\frac{\log(4/\delta)}{2nK}},
\qquad
B_\theta(\tau)
:=
(1+2\beta\tau)\alpha_{n,K,\delta}+\beta T_\theta(\tau),
\]
where
\[
T_\theta(\tau)
:=
\E_{(X,Y)\sim D_\theta}\big[(|\ell_\theta(X,Y)|-\tau)_+\big].
\]
Let $(X,Y)\sim D_\theta$ and define the nonnegative random variable $Z:=|\ell_\theta(X,Y)|$.
With this notation one has $T_\theta(\tau)=\E[(Z-\tau)_+]$, so the function of interest can be written as
\[
B_\theta(\tau)=(1+2\beta\tau)\alpha_{n,K,\delta}+\beta\,\E[(Z-\tau)_+].
\]

The next step is to relate the one-sided derivatives of $\tau\mapsto \E[(Z-\tau)_+]$ to the tail probabilities of $Z$.
For every $z\ge 0$ and every $\tau\ge 0$, the identity
\[
(z-\tau)_+ \;=\; \int_\tau^\infty \mathbf{1}\{z>t\}\,dt
\]
holds, because the integrand equals $1$ precisely on the interval $t\in[\tau,z)$ when $z>\tau$, and otherwise it is identically zero.
Applying this identity with $z=Z$ and using Tonelli's theorem, which is applicable because the integrand is nonnegative, yields the representation
\[
\E[(Z-\tau)_+]
=
\int_\tau^\infty \Pr(Z>t)\,dt.
\]

Let $\tau\ge 0$ and let $h>0$.
Using the integral representation at $\tau$ and at $\tau+h$ gives
\[
\E[(Z-(\tau+h))_+] - \E[(Z-\tau)_+]
=
-\int_\tau^{\tau+h}\Pr(Z>t)\,dt.
\]
Since the function $t\mapsto \Pr(Z>t)$ is nonincreasing, one has
\[
h\,\Pr(Z>\tau+h)\;\le\;\int_\tau^{\tau+h}\Pr(Z>t)\,dt\;\le\;h\,\Pr(Z>\tau).
\]
Dividing by $h$ and combining with the previous display yields
\[
-\Pr(Z>\tau)\;\le\;\frac{\E[(Z-(\tau+h))_+] - \E[(Z-\tau)_+]}{h}\;\le\;-\Pr(Z>\tau+h).
\]
Letting $h\downarrow 0$ and using the monotone convergence $\Pr(Z>\tau+h)\to\Pr(Z>\tau)$ yields the right derivative identity
\[
\frac{d}{d\tau^+}\E[(Z-\tau)_+] \;=\; -\Pr(Z>\tau).
\]

Let $\tau>0$ and let $h\in(0,\tau)$.
Using the integral representation at $\tau$ and at $\tau-h$ gives
\[
\E[(Z-\tau)_+] - \E[(Z-(\tau-h))_+]
=
-\int_{\tau-h}^{\tau}\Pr(Z>t)\,dt.
\]
Since $t\mapsto \Pr(Z>t)$ is nonincreasing, one has
\[
h\,\Pr(Z>\tau)\;\le\;\int_{\tau-h}^{\tau}\Pr(Z>t)\,dt\;\le\;h\,\Pr(Z>\tau-h).
\]
Dividing by $h$ and combining with the previous display yields
\[
-\Pr(Z>\tau-h)\;\le\;\frac{\E[(Z-\tau)_+] - \E[(Z-(\tau-h))_+]}{h}\;\le\;-\Pr(Z>\tau).
\]
Letting $h\downarrow 0$ and using the monotone convergence $\Pr(Z>\tau-h)\to\Pr(Z\ge\tau)$ yields the left derivative identity
\[
\frac{d}{d\tau^-}\E[(Z-\tau)_+] \;=\; -\Pr(Z\ge\tau).
\]

It now follows that $B_\theta$ has one-sided derivatives for every $\tau\ge 0$, and these derivatives satisfy
\[
B_\theta'(\tau^+)
=
2\beta\alpha_{n,K,\delta}-\beta\Pr(Z>\tau),
\qquad
B_\theta'(\tau^-)
=
2\beta\alpha_{n,K,\delta}-\beta\Pr(Z\ge\tau)
\quad\text{for every }\tau>0.
\]

Let $\tau^\star$ be any minimizer of $\tau\mapsto B_\theta(\tau)$ over $\tau\ge 0$.
If $\tau^\star>0$, the minimality of $\tau^\star$ implies that the left derivative is nonpositive and the right derivative is nonnegative, so $B_\theta'((\tau^\star)^-)\le 0\le B_\theta'((\tau^\star)^+)$ holds.
Substituting the one-sided derivative expressions yields
\[
\Pr(Z>\tau^\star)\;\le\;2\alpha_{n,K,\delta}\;\le\;\Pr(Z\ge\tau^\star).
\]
If $\tau^\star=0$, the minimality of $\tau^\star$ implies $0\le B_\theta'(0^+)$, and therefore $\Pr(Z>0)\le 2\alpha_{n,K,\delta}$ holds.
If $2\alpha_{n,K,\delta}<1$, the inequality $2\alpha_{n,K,\delta}\le \Pr(Z\ge 0)=1$ holds as well, and this yields the same two-sided condition with $\tau^\star=0$.

Finally, if $2\alpha_{n,K,\delta}\ge 1$, for every $\tau>0$, one has
\[
B_\theta'(\tau^-)
=
2\beta\alpha_{n,K,\delta}-\beta\Pr(Z\ge\tau)
\ge
2\beta\alpha_{n,K,\delta}-\beta
\ge
0,
\]
and therefore $B_\theta$ is nondecreasing on $(0,\infty)$, which implies that $\tau^\star=0$ is a minimizer over $\tau\ge 0$.
Recalling that $Z=|\ell_\theta(X,Y)|$ with $(X,Y)\sim D_\theta$, the stated conditions are exactly
\[
\Pr_{(X,Y)\sim D_\theta}\!\big(|\ell_\theta(X,Y)|>\tau^\star\big)
\;\le\;
2\alpha_{n,K,\delta}
\;\le\;
\Pr_{(X,Y)\sim D_\theta}\!\big(|\ell_\theta(X,Y)|\ge\tau^\star\big),
\]
and when $\Pr(Z=\tau^\star)=0$ the two inequalities collapse to the equality $\Pr(Z>\tau^\star)=2\alpha_{n,K,\delta}$, which is equivalent to the quantile statement.
\end{proof}

\subsection{PAC-Bayes auxiliary bounds}
\label{app:proofs:pac_aux}

\begin{lemma}[PAC-Bayes bound for prompt sampling {\citep{mcallester1999modelavg,seeger2002pac}}]
\label{lem:app:pac_prompt}
Let $P$ be a prior distribution on $\Theta$, let $\tau>0$ and $\delta\in(0,1)$ be given, and let $r:\cX\times\cY\to[0,1]$ be a given reward function.
With probability at least $1-\delta$ over $x_1,\dots,x_n\sim\rho$, the following inequality holds simultaneously for all posteriors $Q$ on $\Theta$:
\[
\abs{J^{r,\tau}(Q)-\widehat J_{n,\infty}^{r,\tau}(Q)}
\le
(1+2\beta\tau)\sqrt{\frac{\KL(Q\|P)+\log(4/\delta)}{2n}}.
\]
\end{lemma}

\begin{proof}
Let $\lambda>0$ be arbitrary.
For a given parameter value $\theta\in\Theta$, consider a single prompt draw $X\sim\rho$.
As in the prompt-sampling argument in Lemma~\ref{lem:main:prompt}, the quantity $g_\theta^{r,\tau}(X)$ lies in the interval $[-\beta\tau,\,1+\beta\tau]$.
Consequently, the centered random variable $J^{r,\tau}(\theta)-g_\theta^{r,\tau}(X)$ is almost surely bounded in an interval of width $1+2\beta\tau$.
Applying Lemma~\ref{lem:app:hoeffding_lemma} yields
\[
\E_{X\sim\rho}
\exp\!\left(
\lambda\bigl(J^{r,\tau}(\theta)-g_\theta^{r,\tau}(X)\bigr)
\right)
\le
\exp\!\left(\frac{\lambda^2(1+2\beta\tau)^2}{8}\right).
\]

Now let $x_1,\dots,x_n$ be i.i.d.\ draws from $\rho$.
Using independence and the definition
\[
\widehat J_{n,\infty}^{r,\tau}(\theta)
=
\frac{1}{n}\sum_{i=1}^n g_\theta^{r,\tau}(x_i),
\]
it follows that
\[
\E
\exp\!\left(
\lambda\bigl(J^{r,\tau}(\theta)-\widehat J_{n,\infty}^{r,\tau}(\theta)\bigr)
\right)
\le
\exp\!\left(\frac{\lambda^2(1+2\beta\tau)^2}{8n}\right).
\]
Taking expectation with respect to $\theta\sim P$ and applying Markov's inequality yields that, with probability at least $1-\delta/2$ over $x_{1:n}$,
\[
\E_{\theta\sim P}
\exp\!\left(
\lambda\bigl(J^{r,\tau}(\theta)-\widehat J_{n,\infty}^{r,\tau}(\theta)\bigr)
\right)
\le
\frac{2}{\delta}
\exp\!\left(\frac{\lambda^2(1+2\beta\tau)^2}{8n}\right).
\]

On this event, Lemma~\ref{lem:app:com} can be applied with
\[
F(\theta)
=
\lambda\bigl(J^{r,\tau}(\theta)-\widehat J_{n,\infty}^{r,\tau}(\theta)\bigr).
\]
For every posterior $Q$ on $\Theta$, this gives
\begin{align*}
\lambda\bigl(J^{r,\tau}(Q)-\widehat J_{n,\infty}^{r,\tau}(Q)\bigr)
\le{}&
\KL(Q\|P)
+
\log\frac{2}{\delta}
+
\frac{\lambda^2(1+2\beta\tau)^2}{8n}.
\end{align*}
Optimizing over $\lambda>0$ yields the one-sided bound
\[
J^{r,\tau}(Q)-\widehat J_{n,\infty}^{r,\tau}(Q)
\le
(1+2\beta\tau)\sqrt{\frac{\KL(Q\|P)+\log(2/\delta)}{2n}}.
\]
Applying the same argument to the opposite deviation $\widehat J_{n,\infty}^{r,\tau}(Q)-J^{r,\tau}(Q)$ and taking a union bound yields the stated two-sided inequality with $\log(4/\delta)$.
\end{proof}

\begin{lemma}[PAC-Bayes bound for rollout sampling {\citep{catoni2007pac}}]
\label{lem:app:pac_rollout}
Let $P$ be a prior distribution on $\Theta$, let $\tau>0$ and $\delta\in(0,1)$ be given, and let $r:\cX\times\cY\to[0,1]$ be a given reward function.
With probability at least $1-\delta$ over the rollouts conditional on $x_{1:n}$, the following inequality holds simultaneously for all posteriors $Q$ on $\Theta$:
\[
\abs{\widehat J_{n,\infty}^{r,\tau}(Q)-\widehat J_{n,K}^{r,\tau}(Q)}
\le
(1+2\beta\tau)\sqrt{\frac{\KL(Q\|P)+\log(4/\delta)}{2nK}}.
\]
\end{lemma}

\begin{proof}
Condition on the realized prompts $x_{1:n}$, and let $\lambda>0$ be arbitrary.
For each index pair $(i,j)$, define
\[
Z_{i,j}(\theta)
:=
r(x_i,y_{i,j})
-
\beta\,\ell_\theta^\tau(x_i,y_{i,j}).
\]
For every $\theta\in\Theta$, the bounds $r\in[0,1]$ and $\ell_\theta^\tau\in[-\tau,\tau]$ imply
\[
-\beta\tau
\le
Z_{i,j}(\theta)
\le
1+\beta\tau.
\]
Conditional on $(x_{1:n},\theta)$, the rollouts are independent across all pairs $(i,j)$.

Define the deviation
\[
\Delta(\theta)
:=
\widehat J_{n,\infty}^{r,\tau}(\theta)-\widehat J_{n,K}^{r,\tau}(\theta).
\]
By construction, $\widehat J_{n,K}^{r,\tau}(\theta)$ is the average of the $nK$ random variables $Z_{i,j}(\theta)$, and $\widehat J_{n,\infty}^{r,\tau}(\theta)$ is their conditional expectation given $x_{1:n}$.
Applying Lemma~\ref{lem:app:hoeffding_lemma} to the average of bounded independent terms yields
\[
\E\!\left[\exp\bigl(\lambda\Delta(\theta)\bigr)\mid x_{1:n},\theta\right]
\le
\exp\!\left(\frac{\lambda^2(1+2\beta\tau)^2}{8nK}\right).
\]

Taking expectation over $\theta\sim P$ and applying Markov's inequality implies that, with probability at least $1-\delta/2$ over rollouts conditional on $x_{1:n}$,
\[
\E_{\theta\sim P}
\left[
\exp\bigl(\lambda\Delta(\theta)\bigr)
\mid x_{1:n}
\right]
\le
\frac{2}{\delta}
\exp\!\left(\frac{\lambda^2(1+2\beta\tau)^2}{8nK}\right).
\]
On this event, Lemma~\ref{lem:app:com} applied with $F(\theta)=\lambda\Delta(\theta)$ yields that, for every posterior $Q$,
\[
\lambda\bigl(\widehat J_{n,\infty}^{r,\tau}(Q)-\widehat J_{n,K}^{r,\tau}(Q)\bigr)
\le
\KL(Q\|P)
+
\log\frac{2}{\delta}
+
\frac{\lambda^2(1+2\beta\tau)^2}{8nK}.
\]
Optimizing over $\lambda>0$ gives
\[
\widehat J_{n,\infty}^{r,\tau}(Q)-\widehat J_{n,K}^{r,\tau}(Q)
\le
(1+2\beta\tau)\sqrt{\frac{\KL(Q\|P)+\log(2/\delta)}{2nK}}.
\]
Applying the same argument to the deviation $-\Delta(\theta)$ and taking a union bound yields the stated two-sided inequality with $\log(4/\delta)$.
\end{proof}

\subsection{PAC-Bayes main bound}
\label{app:proofs:pac_main}

\begin{proof}[Proof of Theorem~\ref{thm:main:pac}]
Let $\phi\in\Phi$ be arbitrary, and let $\tau>0$ and $\delta\in(0,1)$ be given.
Let $P$ denote the prior that appears in Theorem~\ref{thm:main:pac}.
The proof proceeds by combining two PAC-Bayes concentration inequalities with the deterministic reward-shift and clipping-bias bounds, and then substituting these ingredients into the same three-term decomposition used in the fixed-policy case.

Apply Lemma~\ref{lem:app:pac_rollout} with reward $r=\hat r_\phi$ and confidence level $\delta/2$.
Apply Lemma~\ref{lem:app:pac_prompt} with reward $r=\hat r_\phi$ and confidence level $\delta/2$.
By a union bound, with probability at least $1-\delta$ over prompts and rollouts, both inequalities hold simultaneously for all posteriors $Q$ on $\Theta$.

On this event, for every posterior $Q$,
\begin{align*}
\abs{\widehat J_{n,K}^{\phi,\tau}(Q)-\widehat J_{n,\infty}^{\phi,\tau}(Q)}
\le{}&
(1+2\beta\tau)\sqrt{\frac{\KL(Q\|P)+\log(8/\delta)}{2nK}},\\
\abs{\widehat J_{n,\infty}^{\phi,\tau}(Q)-J^{\phi,\tau}(Q)}
\le{}&
(1+2\beta\tau)\sqrt{\frac{\KL(Q\|P)+\log(8/\delta)}{2n}}.
\end{align*}
Combining these two bounds via the triangle inequality yields
\begin{align*}
\abs{\widehat J_{n,K}^{\phi,\tau}(Q)-J^{\phi,\tau}(Q)}
\le{}&
(1+2\beta\tau)\left(
\sqrt{\frac{\KL(Q\|P)+\log(8/\delta)}{2n}}
+
\sqrt{\frac{\KL(Q\|P)+\log(8/\delta)}{2nK}}
\right).
\end{align*}

The remaining two contributions follow by averaging pointwise bounds over $\theta\sim Q$.
Taking expectation in Lemma~\ref{lem:main:bias} yields
\[
\abs{J^{\phi,\tau}(Q)-J^\phi(Q)}
\le
\beta\,\E_{\theta\sim Q}\!\left[
\E_{(X,Y)\sim D_\theta}
\!\left[
\abs{\ell_\theta(X,Y)-\ell_\theta^\tau(X,Y)}
\right]
\right].
\]
Taking expectation in Lemma~\ref{lem:main:shift} yields
\[
\abs{J^\phi(Q)-J^\star(Q)}
\le
\E_{\theta\sim Q}[\Ccov(\theta)]\,\sqrt{L_{\mathrm{train}}^{(2)}(\phi)}.
\]

Finally, apply the same add-and-subtract decomposition used in Lemma~\ref{prop:main:decomp} directly to $\widehat J_{n,K}^{\phi,\tau}(Q)-J^\star(Q)$, and then substitute the three bounds established above to obtain the stated inequality.
On the same event of probability at least $1-\delta$, this gives the inequality stated in Theorem~\ref{thm:main:pac}, and the statement holds simultaneously for all posteriors $Q$ because the concentration step was uniform over $Q$.
\end{proof}

\subsection{Proofs for PAC-Bayes special cases}
\label{app:proofs:pac_special}

\subsubsection{Finite candidate class and checkpoint selection}
\label{app:proofs:finite_class}

\begin{proof}[Proof of Corollary~\ref{cor:main:finite_class}]
Let $M\ge 2$ be an integer, and let $\Theta_M=\{\theta^{(1)},\dots,\theta^{(M)}\}$ be the finite set of candidate parameters described in the statement of the corollary.
Let $P$ denote the uniform distribution on $\Theta_M$, so that $P(\theta^{(m)})=1/M$ holds for every $m\in\{1,\dots,M\}$.
Let $Q$ be an arbitrary distribution supported on the same finite set $\Theta_M$.

For each $m\in\{1,\dots,M\}$, define
\[
p_m := P(\theta^{(m)})=\frac{1}{M},
\qquad
q_m := Q(\theta^{(m)}),
\]
so that $q_m\ge 0$ holds for every $m$ and $\sum_{m=1}^M q_m=1$ holds by the definition of a probability mass function.
By the definition of the Kullback--Leibler divergence on a finite set, one has
\[
\KL(Q\|P)
=
\sum_{m=1}^M q_m \log\frac{q_m}{p_m}.
\]
Substituting the identity $p_m=1/M$ into the preceding display yields
\[
\KL(Q\|P)
=
\sum_{m=1}^M q_m \log(q_m M)
=
\log M + \sum_{m=1}^M q_m \log q_m,
\]
where the final equality follows because $\sum_{m=1}^M q_m=1$ allows the factor $\log M$ to be separated from the summation.

It therefore remains to control the quantity $\sum_{m=1}^M q_m \log q_m$.
For every index $m\in\{1,\dots,M\}$, the probability value $q_m$ lies in the interval $[0,1]$, and therefore one has $\log q_m\le 0$ whenever $q_m>0$, which implies that $q_m\log q_m\le 0$ whenever $q_m>0$.
When $q_m=0$, the contribution $q_m\log q_m$ is interpreted as $0$, which is consistent with the limiting identity $\lim_{t\downarrow 0} t\log t=0$.
Consequently, every term in the sum $\sum_{m=1}^M q_m \log q_m$ is less than or equal to $0$, and hence
\[
\sum_{m=1}^M q_m \log q_m \le 0.
\]
Substituting this inequality into the identity above gives
\[
\KL(Q\|P)
=
\log M + \sum_{m=1}^M q_m \log q_m
\le
\log M.
\]

Finally, consider the special case in which $Q$ is the Dirac distribution concentrated on a single element $\theta^{(\widehat m)}\in\Theta_M$.
In that case one has $q_{\widehat m}=1$ and $q_m=0$ for all $m\neq \widehat m$.
Substituting these values into the definition
\(
\KL(Q\|P)=\sum_{m=1}^M q_m \log\frac{q_m}{p_m}
\)
shows that the only nonzero contribution is the term indexed by $\widehat m$, and therefore
\[
\KL(Q\|P)
=
1\cdot \log\frac{1}{1/M}
=
\log M.
\]
This proves the final statement of the corollary.
\end{proof}

\subsubsection{OU--SGD special case for the PAC-Bayes complexity term}
\label{app:proofs:ou_sgd}

\begin{lemma}[Bounds for the stationary covariance in the OU approximation]
\label{lem:app:ou_cov_bounds}
Let $H\in\mathbb{R}^{d\times d}$ be symmetric and positive definite, let $\Sigma_g\in\mathbb{R}^{d\times d}$ be symmetric and positive definite, and let $\varepsilon>0$.
Assume that $\Sigma\in\mathbb{R}^{d\times d}$ is symmetric and satisfies the matrix equation
\[
H\Sigma+\Sigma H=\varepsilon\,\Sigma_g.
\]
Assume also that $H$ and $\Sigma_g$ commute, meaning that $H\Sigma_g=\Sigma_g H$ holds.
Assume finally that there exist constants $0<m\le M<\infty$ such that $mI\preceq H\preceq MI$.
Then, $\Sigma$ satisfies the two-sided bound
\begin{equation}
\label{eq:app:ou_cov_bounds}
\begin{aligned}
\frac{\varepsilon}{2M}\,\Sigma_g
\preceq
\Sigma
\preceq
\frac{\varepsilon}{2m}\,\Sigma_g.
\end{aligned}
\end{equation}
\end{lemma}

\begin{proof}
Throughout the proof, for symmetric matrices $A$ and $B$, the notation $A\preceq B$ means that $v^\top A v\le v^\top B v$ holds for every vector $v\in\mathbb{R}^d$.
This definition is convenient because it reduces the verification of a matrix inequality to the verification of an ordinary inequality that holds uniformly over all vectors.

Define the matrix-valued function
\[
F(t):=e^{-tH}\,\Sigma\,e^{-tH}
\qquad\text{for }t\ge 0.
\]
Since $H$ is symmetric, the matrix exponential $e^{-tH}$ is well-defined for every $t\ge 0$, and the map $t\mapsto F(t)$ is differentiable.
Differentiating and using the product rule yields
\[
\frac{d}{dt}F(t)
=
(-He^{-tH})\Sigma e^{-tH}
+
e^{-tH}\Sigma(-He^{-tH})
=
-e^{-tH}(H\Sigma+\Sigma H)e^{-tH}.
\]
Substituting the identity $H\Sigma+\Sigma H=\varepsilon\,\Sigma_g$ gives
\[
\frac{d}{dt}F(t)
=
-\varepsilon\,e^{-tH}\Sigma_g e^{-tH}.
\]
Integrating the preceding identity from $0$ to $T$ gives
\[
F(T)-F(0)
=
-\varepsilon\int_0^T e^{-tH}\Sigma_g e^{-tH}\,dt.
\]
Since $F(0)=\Sigma$, rearranging yields
\[
\Sigma
=
F(T)+\varepsilon\int_0^T e^{-tH}\Sigma_g e^{-tH}\,dt.
\]
Because $H$ is positive definite, there exists a constant $m_0>0$ such that $H\succeq m_0 I$, and therefore the operator norm satisfies $\|e^{-tH}\|_2\le e^{-tm_0}$ for every $t\ge 0$.
This inequality implies $\|F(T)\|_2=\|e^{-TH}\Sigma e^{-TH}\|_2\le \|e^{-TH}\|_2^2\|\Sigma\|_2\le e^{-2Tm_0}\|\Sigma\|_2$, and hence $F(T)$ converges to the zero matrix as $T\to\infty$.
Taking the limit $T\to\infty$ yields the integral identity
\[
\Sigma
=
\varepsilon\int_0^\infty e^{-tH}\Sigma_g e^{-tH}\,dt.
\]

It remains to compare $e^{-tH}\Sigma_g e^{-tH}$ to scalar multiples of $\Sigma_g$ in the Loewner order.
The assumption $mI\preceq H\preceq MI$ means that every eigenvalue of $H$ lies in the interval $[m,M]$.
Consequently, every eigenvalue of $e^{-2tH}$ lies in the interval $[e^{-2tM},\,e^{-2tm}]$, and this implies the inequalities
\[
e^{-2tM}I \preceq e^{-2tH} \preceq e^{-2tm}I
\qquad\text{for every }t\ge 0.
\]
The commutativity condition $H\Sigma_g=\Sigma_g H$ implies that $\Sigma_g$ commutes with the matrix exponential $e^{-tH}$ for every $t\ge 0$.
Therefore one has
\[
e^{-tH}\Sigma_g e^{-tH}
=
\Sigma_g e^{-tH}e^{-tH}
=
\Sigma_g e^{-2tH}.
\]
Since $\Sigma_g\succ 0$, the matrix square root $\Sigma_g^{1/2}$ exists and is symmetric and positive definite.
Applying the congruence transformation with $\Sigma_g^{1/2}$ to the Loewner inequalities above yields
\[
\Sigma_g^{1/2}\bigl(e^{-2tM}I\bigr)\Sigma_g^{1/2}
\preceq
\Sigma_g^{1/2}e^{-2tH}\Sigma_g^{1/2}
\preceq
\Sigma_g^{1/2}\bigl(e^{-2tm}I\bigr)\Sigma_g^{1/2}
\qquad\text{for every }t\ge 0.
\]
Using $\Sigma_g^{1/2}I\Sigma_g^{1/2}=\Sigma_g$ and the scalar factors in the two outer terms gives
\[
e^{-2tM}\Sigma_g
\preceq
\Sigma_g^{1/2}e^{-2tH}\Sigma_g^{1/2}
\preceq
e^{-2tm}\Sigma_g
\qquad\text{for every }t\ge 0.
\]
The commutativity condition implies that $\Sigma_g^{1/2}$ commutes with $e^{-tH}$ and therefore also commutes with $e^{-2tH}$, which yields
\[
\Sigma_g^{1/2}e^{-2tH}\Sigma_g^{1/2}
=
e^{-2tH}\Sigma_g
=
e^{-tH}\Sigma_g e^{-tH}.
\]
Substituting this identity into the preceding display yields
\[
e^{-2tM}\Sigma_g
\preceq
e^{-tH}\Sigma_g e^{-tH}
\preceq
e^{-2tm}\Sigma_g
\qquad\text{for every }t\ge 0.
\]

Substituting these two bounds into the integral representation of $\Sigma$ yields
\[
\varepsilon\int_0^\infty e^{-2tM}\Sigma_g\,dt
\preceq
\Sigma
\preceq
\varepsilon\int_0^\infty e^{-2tm}\Sigma_g\,dt.
\]
Evaluating the scalar integrals gives
\[
\varepsilon\int_0^\infty e^{-2tM}\,dt=\frac{\varepsilon}{2M},
\qquad
\varepsilon\int_0^\infty e^{-2tm}\,dt=\frac{\varepsilon}{2m},
\]
and substituting these values proves eq. \eqref{eq:app:ou_cov_bounds}.
\end{proof}

\begin{proof}[Proof of Corollary~\ref{cor:main:ou_sgd}]
Assume the parameter space is $\mathbb{R}^d$ and the prior is $P=\mathcal{N}(\theta_0,\Lambda)$ with $\Lambda\succ 0$.
Furthermore, assume the posterior induced by SGD with constant step size $\varepsilon>0$ is approximated by the stationary Ornstein-Uhlenbeck law $Q_{\mathrm{SGD}}=\mathcal{N}(\hat\theta,\Sigma)$. 

By the local quadratic approximation of the objective, the covariance $\Sigma$ satisfies the continuous Lyapunov equation $H\Sigma+\Sigma H=\varepsilon\,\Sigma_g$, where $\Sigma_g\succ 0$ is the gradient noise covariance and $H\succ 0$ is the objective Hessian at the optimum $\hat\theta$. We assume that $H$ and $\Sigma_g$ commute, and that the matrix $H$ is symmetric and satisfies $mI\preceq H\preceq MI$ for some constants $0<m\le M<\infty$.

Apply Lemma~\ref{lem:app:ou_gauss_kl} with $\mu_Q=\hat\theta$, $\Sigma_Q=\Sigma$, $\mu_P=\theta_0$, and $\Sigma_P=\Lambda$.
This yields
\begin{equation}
\label{eq:app:ou_kl_start}
\begin{aligned}
\KL(Q_{\mathrm{SGD}}\|P)
=
\frac12\Big(
\tr(\Lambda^{-1}\Sigma)
+
(\hat\theta-\theta_0)^\top\Lambda^{-1}(\hat\theta-\theta_0)
-d
+
\log\frac{\det(\Lambda)}{\det(\Sigma)}
\Big).
\end{aligned}
\end{equation}
The remaining task is to upper bound the trace term and to upper bound the logarithmic determinant ratio in a way that makes the dependence on $\varepsilon$, $\Sigma_g$, and the constants $m$ and $M$ explicit.

First, apply Lemma~\ref{lem:app:ou_cov_bounds}, which gives $\Sigma\preceq \frac{\varepsilon}{2m}\Sigma_g$.
Since $\Lambda^{-1}\succ 0$, this inequality implies $\Lambda^{-1/2}\Sigma\Lambda^{-1/2}\preceq \frac{\varepsilon}{2m}\Lambda^{-1/2}\Sigma_g\Lambda^{-1/2}$, and taking traces yields
\[
\tr(\Lambda^{-1}\Sigma)
=
\tr(\Lambda^{-1/2}\Sigma\Lambda^{-1/2})
\le
\frac{\varepsilon}{2m}\tr(\Lambda^{-1/2}\Sigma_g\Lambda^{-1/2})
=
\frac{\varepsilon}{2m}\tr(\Lambda^{-1}\Sigma_g).
\]

Second, apply Lemma~\ref{lem:app:ou_cov_bounds} again, which also gives $\Sigma\succeq \frac{\varepsilon}{2M}\Sigma_g$.
This inequality implies that the eigenvalues of $\Sigma$ dominate the eigenvalues of $\frac{\varepsilon}{2M}\Sigma_g$ when both collections are arranged in nondecreasing order, and therefore the product of the eigenvalues of $\Sigma$ is at least the product of the eigenvalues of $\frac{\varepsilon}{2M}\Sigma_g$.
Consequently,
\[
\det(\Sigma)
\ge
\det\!\left(\frac{\varepsilon}{2M}\Sigma_g\right)
=
\left(\frac{\varepsilon}{2M}\right)^d \det(\Sigma_g),
\]
where the last equality uses the basic scaling rule for determinants.
Taking logarithms and rearranging yields
\[
\log\frac{\det(\Lambda)}{\det(\Sigma)}
\le
\log\det(\Lambda)-\log\det(\Sigma_g)-d\log\!\left(\frac{\varepsilon}{2M}\right).
\]

Substituting the preceding two bounds into eq. \eqref{eq:app:ou_kl_start} yields the claimed inequality eq. \eqref{eq:main:ou_sgd_kl}, and this completes the proof.
\end{proof}

\subsection{Budget allocation}
\label{app:proofs:budget}

\begin{proof}[Derivation of the uniform-cost baseline $K^\star=1$]
Assume that the sampling budget satisfies $nK\le B$ for some $B>0$, and assume that each rollout has the same cost so that the constraint depends only on the product $nK$.
Consider the leading-order sampling structure in Lemma~\ref{lem:main:stat} and ignore multiplicative constants that do not depend on $K$.
The resulting proxy has the form
\[
\frac{1}{\sqrt{n}}+\frac{1}{\sqrt{nK}}.
\]
Under the constraint $nK\le B$, one may take $n=B/K$ without loss of generality for minimizing the proxy over $K\ge 1$.
Substituting $n=B/K$ yields
\[
\frac{1}{\sqrt{n}}+\frac{1}{\sqrt{nK}}
=
\sqrt{\frac{K}{B}}+\frac{1}{\sqrt{B}}.
\]
The second term does not depend on $K$, and the first term is strictly increasing in $K$ for $K\ge 1$.
Therefore the proxy is minimized by the smallest admissible value of $K$, which is $K^\star=1$.
\end{proof}

\begin{proof}[Proof of Corollary~\ref{cor:opt_K_hoeffding}]
Let $B>0$, $c_{\mathrm{prefill}}>0$, and $c_{\mathrm{decode}}>0$ be given.
Assume the budget constraint
\[
B \ge n\,c_{\mathrm{prefill}} + nK\,c_{\mathrm{decode}},
\]
and consider the leading-order sampling structure induced by Lemma~\ref{lem:main:stat}.
As in the statement, treat $K$ as a continuous variable with $K\ge 1$ and ignore multiplicative constants and logarithmic factors that do not depend on $K$.
The sampling proxy can be written in the form
\[
E(n,K)
=
\frac{1}{\sqrt{n}}+\frac{1}{\sqrt{nK}}.
\]
Under the constraint, the choice
\[
n=\frac{B}{c_{\mathrm{prefill}}+Kc_{\mathrm{decode}}}
\]
saturates the budget and maximizes $n$ for a given $K$, hence it minimizes $E(n,K)$ for that $K$.
Substituting this expression for $n$ gives an objective that depends only on $K$,
\[
E(K)
=
\sqrt{\frac{c_{\mathrm{prefill}}+Kc_{\mathrm{decode}}}{B}}
\left(1+\frac{1}{\sqrt{K}}\right).
\]
Since $B$ is constant, minimizing $E(K)$ over $K\ge 1$ is equivalent to minimizing the squared objective
\[
F(K)
:=
\bigl(c_{\mathrm{prefill}}+Kc_{\mathrm{decode}}\bigr)
\left(1+\frac{1}{\sqrt{K}}\right)^2.
\]
Expanding the square gives
\[
F(K)
=
\bigl(c_{\mathrm{prefill}}+Kc_{\mathrm{decode}}\bigr)
\left(1+\frac{2}{\sqrt{K}}+\frac{1}{K}\right)
=
\bigl(c_{\mathrm{prefill}}+Kc_{\mathrm{decode}}\bigr)
+
2\bigl(c_{\mathrm{prefill}}+Kc_{\mathrm{decode}}\bigr)K^{-1/2}
+
\bigl(c_{\mathrm{prefill}}+Kc_{\mathrm{decode}}\bigr)K^{-1}.
\]
Differentiating term by term yields
\[
F'(K)
=
c_{\mathrm{decode}}
+
2\left(
c_{\mathrm{decode}}K^{-1/2}
-\frac{1}{2}\bigl(c_{\mathrm{prefill}}+Kc_{\mathrm{decode}}\bigr)K^{-3/2}
\right)
+
\left(
c_{\mathrm{decode}}K^{-1}
-\bigl(c_{\mathrm{prefill}}+Kc_{\mathrm{decode}}\bigr)K^{-2}
\right).
\]
Simplifying this expression gives
\[
F'(K)
=
c_{\mathrm{decode}}
+
c_{\mathrm{decode}}K^{-1/2}
-
c_{\mathrm{prefill}}K^{-3/2}
-
c_{\mathrm{prefill}}K^{-2}.
\]
Multiplying by $K^2$ yields an equivalent first-order condition
\[
K^2 F'(K)
=
c_{\mathrm{decode}}K^2
+
c_{\mathrm{decode}}K^{3/2}
-
c_{\mathrm{prefill}}K^{1/2}
-
c_{\mathrm{prefill}}.
\]
Let $u=\sqrt{K}$, so that $K=u^2$ and $K^{3/2}=u^3$.
The condition $F'(K)=0$ is equivalent to
\[
c_{\mathrm{decode}}u^4 + c_{\mathrm{decode}}u^3 - c_{\mathrm{prefill}}u - c_{\mathrm{prefill}} = 0,
\]
and the polynomial factors as
\[
c_{\mathrm{decode}}u^3(u+1) - c_{\mathrm{prefill}}(u+1) = (u+1)\bigl(c_{\mathrm{decode}}u^3-c_{\mathrm{prefill}}\bigr).
\]
Since $u=\sqrt{K}\ge 1$, one has $u+1>0$, so any interior stationary point satisfies
$c_{\mathrm{decode}}u^3=c_{\mathrm{prefill}}$. Therefore
\[
u=\left(\frac{c_{\mathrm{prefill}}}{c_{\mathrm{decode}}}\right)^{1/3},
\qquad
K=u^2=\left(\frac{c_{\mathrm{prefill}}}{c_{\mathrm{decode}}}\right)^{2/3}.
\]
This is the interior stationary point. Because the optimisation domain is $K\ge 1$, the continuous proxy minimiser is
\[
K^\star=\max\!\left\{1,\left(\frac{c_{\mathrm{prefill}}}{c_{\mathrm{decode}}}\right)^{2/3}\right\}.
\]
\end{proof}

\begin{proof}[Proof of Corollary~\ref{cor:opt_K_variance}]
Let $Z$ denote the per-rollout contribution used in the empirical objective.
Assume that the estimator averages $Z$ over $n$ independent prompts and $K$ independent rollouts per prompt.
Define the two-stage variance quantities
\[
\sigma_{\mathrm{prompt}}^2:=\mathrm{Var}\!\left(\E[Z\mid X]\right),
\qquad
\sigma_{\mathrm{rollout}}^2:=\E\!\left[\mathrm{Var}(Z\mid X)\right].
\]
The standard variance decomposition for a two-stage average yields the proxy
\[
V(n,K)
=
\frac{\sigma_{\mathrm{prompt}}^2}{n}
+
\frac{\sigma_{\mathrm{rollout}}^2}{nK}.
\]
Assume the budget constraint
\[
B \ge n\,c_{\mathrm{prefill}} + nK\,c_{\mathrm{decode}},
\]
and substitute the saturated choice $n=B/(c_{\mathrm{prefill}}+Kc_{\mathrm{decode}})$.
This yields
\[
V(K)
=
\frac{c_{\mathrm{prefill}}+Kc_{\mathrm{decode}}}{B}
\left(
\sigma_{\mathrm{prompt}}^2+\frac{\sigma_{\mathrm{rollout}}^2}{K}
\right).
\]
Since $B$ is constant, minimizing $V(K)$ over $K\ge 1$ is equivalent to minimizing
\[
G(K)
:=
\bigl(c_{\mathrm{prefill}}+Kc_{\mathrm{decode}}\bigr)
\left(
\sigma_{\mathrm{prompt}}^2+\frac{\sigma_{\mathrm{rollout}}^2}{K}
\right).
\]
Expanding gives
\[
G(K)
=
c_{\mathrm{prefill}}\sigma_{\mathrm{prompt}}^2
+
c_{\mathrm{prefill}}\frac{\sigma_{\mathrm{rollout}}^2}{K}
+
c_{\mathrm{decode}}K\sigma_{\mathrm{prompt}}^2
+
c_{\mathrm{decode}}\sigma_{\mathrm{rollout}}^2.
\]
Differentiating yields
\[
G'(K)
=
-c_{\mathrm{prefill}}\frac{\sigma_{\mathrm{rollout}}^2}{K^2}
+
c_{\mathrm{decode}}\sigma_{\mathrm{prompt}}^2.
\]
Setting $G'(K)=0$ gives
\[
c_{\mathrm{decode}}\sigma_{\mathrm{prompt}}^2
=
c_{\mathrm{prefill}}\frac{\sigma_{\mathrm{rollout}}^2}{K^2},
\]
which implies
\[
K^2
=
\frac{c_{\mathrm{prefill}}}{c_{\mathrm{decode}}}
\cdot
\frac{\sigma_{\mathrm{rollout}}^2}{\sigma_{\mathrm{prompt}}^2}.
\]
Taking square roots yields the interior stationary point
\[
K=\sqrt{\frac{c_{\mathrm{prefill}}}{c_{\mathrm{decode}}}\cdot
\frac{\sigma_{\mathrm{rollout}}^2}{\sigma_{\mathrm{prompt}}^2}}.
\]
Because the optimisation domain is $K\ge 1$, the continuous proxy minimiser is
\[
K^\star=\max\!\left\{1,\sqrt{\frac{c_{\mathrm{prefill}}}{c_{\mathrm{decode}}}\cdot
\frac{\sigma_{\mathrm{rollout}}^2}{\sigma_{\mathrm{prompt}}^2}}\right\}.
\]
\end{proof}

\end{document}